\documentclass{article}

\usepackage{natbib}

\usepackage[final]{neurips_2021}

\usepackage[utf8]{inputenc} 
\usepackage[T1]{fontenc}    
\usepackage[colorlinks=true, linkcolor=blue, citecolor=blue,urlcolor=black]{hyperref}
\usepackage{url}            
\usepackage{booktabs}       
\usepackage{amsfonts}       
\usepackage{nicefrac}       
\usepackage{microtype}      
\usepackage{xcolor}         
\usepackage{graphicx}

\definecolor{Green}{rgb}{0.13, 0.65, 0.3}
\usepackage[]{color-edits}
\addauthor{HL}{Green}
\addauthor{CL}{violet}
\addauthor{ck}{red}

\title{Last-iterate Convergence in Extensive-Form Games}

\author{%
Chung-Wei Lee\\
University of Southern California\\
\texttt{leechung@usc.edu} 
\And
Christian Kroer\\
Columbia University\\
\texttt{christian.kroer@columbia.edu} 
\And
Haipeng Luo\\
University of Southern California\\
\texttt{haipengl@usc.edu}}
\usepackage{xspace}
\usepackage{amsmath}
\usepackage{amsthm}
\usepackage{bbm}
\usepackage{algorithm}
\usepackage[algo2e, vlined, ruled]{algorithm2e}
\usepackage{algorithmic}
\usepackage{mathrsfs}
\usepackage{amssymb}
\usepackage{bm}
\usepackage{makecell}
\usepackage{tabulary}
\usepackage{xcolor}
\usepackage{comment}

\definecolor{Green}{rgb}{0.13, 0.65, 0.3}

\DeclareMathOperator*{\argmax}{argmax}
\DeclareMathOperator*{\argmin}{argmin}
\newcommand{\KLD}{\text{\rm KL}}
\newcommand{\KL}{D_\psi}

\newcommand{\Reg}{\text{\rm Reg}}
\newcommand{\Regp}{\widehat{\text{\rm Reg}}}
\newcommand{\reg}{\text{\rm reg}}

\newcommand{\calV}{\mathcal{V}}

\newcommand{\calZ}{\mathcal{Z}}
\newcommand{\calH}{\mathcal{H}}

\newcommand{\be}{\mathbf{e}}

\newcommand{\order}{\mathcal{O}}

\newcommand{\calX}{\mathcal{X}}
\newcommand{\calY}{\mathcal{Y}}
\newcommand{\one}{\mathbbm{1}}
\newcommand{\inner}[1]{\langle#1\rangle}

\newcommand{\supp}{\text{supp}}

\newcommand{\dist}{\mathrm{dist}^2}

\newcommand{\xp}{\widehat{\bm{x}}}
\newcommand{\yp}{\widehat{\bm{y}}}
\newcommand{\zp}{\widehat{\bm{z}}}
\newcommand{\qp}{\widehat{\bm{q}}}

\newcommand{\zpp}{\widehat{z}}

\newcommand{\qpp}{\widehat{q}}

\newcommand{\modify}[1]{{#1}}

\newcommand{\zero}{\bm{0}}

\newcommand{\bl}{\bm{\ell}}

\newcommand{\w}{{\bm{w}}}
\newcommand{\x}{{\bm{x}}}
\newcommand{\y}{{\bm{y}}}
\newcommand{\q}{{\bm{q}}}
\newcommand{\z}{{\bm{z}}}

\newcommand{\s}{\bm{s}}

\newcommand{\h}{\bm{h}}
\newcommand{\f}{\bm{f}}

\newcommand{\ve}{\bm{e}}
\newcommand{\vf}{\bm{f}}
\newcommand{\bu}{\bm{u}}
\newcommand{\bU}{\bm{U}}
\newcommand{\bv}{\bm{v}}
\newcommand{\bV}{\bm{V}}
\newcommand{\G}{\bm{G}}
\newcommand{\bA}{\bm{A}}
\newcommand{\bB}{\bm{B}}

\newcommand{\diam}{1}
\newcommand{\zmin}{z_{\min}}
\newcommand{\alphamin}{\alpha_{\min}}

\newcommand{\trplx}{\calZ}
\newcommand{\R}{\mathbb{R}}
\newcommand{\brch}[3]{#1~\fbox{\text{$#3$}}~#2}

\newcommand{\balpha}{{\boldsymbol \alpha}}

\newcommand{\simp}{g}
\newcommand{\seqactn}{\Omega}
\newcommand{\csimp}{\calH}
\newcommand{\ca}{\frac{\eta^2\xi\epsd^2}{40P^3\|\balpha\|_\infty}}

\newcommand{\treeX}{{\calX}}
\newcommand{\treeY}{{\calY}}
\newcommand{\infoX}{{\csimp}^{\treeX}}

\newcommand{\lp}{L}
\newcommand{\blp}{\bm{\lp}}
\newcommand{\seq}{\sigma}
\newcommand{\emp}{0}
\newcommand{\matX}{\bA}
\newcommand{\matY}{\bB}

\newcommand{\infoi}{h}
\newcommand{\vent}{\entf^{\text{van}}}
\newcommand{\veuc}{\eucf^{\text{van}}}
\newcommand{\dent}{\entf^{\text{dil}}_{\balpha}}
\newcommand{\deuc}{\eucf^{\text{dil}}_{\balpha}}
\newcommand{\entf}{\Psi}
\newcommand{\eucf}{\Phi}
\newcommand{\epsv}{\epsilon_{{\text{van}}}}
\newcommand{\epsd}{\epsilon_{\text{dil}}}

\newcommand{\oomd}{\textsc{OOMD}\xspace}

\newcommand{\omwu}{\textsc{OMWU}\xspace}
\newcommand{\vogda}{\textsc{VOGDA}\xspace}
\newcommand{\vomwu}{\textsc{VOMWU}\xspace}
\newcommand{\dogda}{\textsc{DOGDA}\xspace}
\newcommand{\domwu}{\textsc{DOMWU}\xspace}

\newcommand{\ODMWU}{\textsc{DOMWU}\xspace}
\newcommand{\breg}[1]{D_{#1}}
\newcommand{\relu}[1]{\left[#1\right]_+}

\newtheorem*{theorem*}{Theorem}
\newtheorem{theorem}{Theorem}

\newtheorem{lemma}[theorem]{Lemma}
\newtheorem{corollary}[theorem]{Corollary}
\newtheorem*{lemma*}{Lemma}
\newtheorem{definition}{Definition}

\theoremstyle{definition}

\usepackage{prettyref}
\newcommand{\pref}[1]{\prettyref{#1}}

\newcommand{\savehyperref}[2]{\texorpdfstring{\hyperref[#1]{#2}}{#2}}
\newrefformat{eq}{\savehyperref{#1}{Eq. \textup{(\ref*{#1})}}}
\newrefformat{eqn}{\savehyperref{#1}{Equation~\ref*{#1}}}
\newrefformat{lem}{\savehyperref{#1}{Lemma~\ref*{#1}}}
\newrefformat{def}{\savehyperref{#1}{Definition~\ref*{#1}}}
\newrefformat{line}{\savehyperref{#1}{Line~\ref*{#1}}}
\newrefformat{thm}{\savehyperref{#1}{Theorem~\ref*{#1}}}
\newrefformat{corr}{\savehyperref{#1}{Corollary~\ref*{#1}}}
\newrefformat{cor}{\savehyperref{#1}{Corollary~\ref*{#1}}}
\newrefformat{sec}{\savehyperref{#1}{Section~\ref*{#1}}}
\newrefformat{app}{\savehyperref{#1}{Appendix~\ref*{#1}}}
\newrefformat{assum}{\savehyperref{#1}{Assumption~\ref*{#1}}}
\newrefformat{ex}{\savehyperref{#1}{Example~\ref*{#1}}}
\newrefformat{fig}{\savehyperref{#1}{Figure~\ref*{#1}}}
\newrefformat{alg}{\savehyperref{#1}{Algorithm~\ref*{#1}}}
\newrefformat{rem}{\savehyperref{#1}{Remark~\ref*{#1}}}
\newrefformat{conj}{\savehyperref{#1}{Conjecture~\ref*{#1}}}
\newrefformat{prop}{\savehyperref{#1}{Proposition~\ref*{#1}}}
\newrefformat{proto}{\savehyperref{#1}{Protocol~\ref*{#1}}}
\newrefformat{prob}{\savehyperref{#1}{Problem~\ref*{#1}}}
\newrefformat{claim}{\savehyperref{#1}{Claim~\ref*{#1}}}
\newrefformat{que}{\savehyperref{#1}{Question~\ref*{#1}}}
\newrefformat{op}{\savehyperref{#1}{Open Problem~\ref*{#1}}}
\begin{document}

\maketitle

\begin{abstract}
   Regret-based algorithms are highly efficient at finding approximate Nash equilibria in sequential games such as poker games.
   However, most regret-based algorithms, including counterfactual regret minimization (CFR) and its variants, rely on iterate averaging to achieve convergence.   
   Inspired by recent advances on \emph{last-iterate} convergence of optimistic algorithms in zero-sum normal-form games, we study this phenomenon in sequential games, and provide a comprehensive study of last-iterate convergence for zero-sum extensive-form games with perfect recall (EFGs), using various optimistic regret-minimization algorithms over treeplexes.
   This includes algorithms using the vanilla entropy or squared Euclidean norm regularizers, as well as their dilated versions which admit more efficient implementation.
   In contrast to CFR, we show that all of these algorithms enjoy last-iterate convergence, with some of them even converging \emph{exponentially} fast. 
   We also provide experiments to further support our theoretical results.
\end{abstract} 

\section{Introduction}\label{sec: intro}
Extensive-form games (EFGs) are an important class of games in game theory and artificial intelligence which can model imperfect information and sequential interactions. 
EFGs are typically solved by finding or approximating a \emph{Nash equilibrium}. 
Regret-minimization algorithms are among the most popular approaches to approximate Nash equilibria.
The motivation comes from a classical result which says that in two-player zero-sum games, when both players use \emph{no-regret} algorithms, \emph{the average strategy} converges to Nash equilibrium~\citep{freund1999adaptive,hart2000simple,zinkevich2007regret}.
Counterfactual Regret Minimization (CFR) \citep{zinkevich2007regret} and it variants such as CFR+ \citep{tammelin2014solving} are based on this motivation.

However, due to their ergodic convergence guarantee, theoretical convergence rates of regret-minimization algorithms are typically limited to $\order(1/\sqrt{T})$ or $\order(1/T)$ for $T$ rounds, and this is also the case in practice \citep{brown2019solving,burch2019revisiting}. 
In contrast, it is known that linear convergence rates are achievable for certain other first-order algorithms~\citep{tseng1995linear,gilpin2008first}.
Additionally, the averaging procedure can create complications.
It not only increases the computational and memory overhead \citep{bowling2015heads}, but also makes things difficult when incorporating neural networks in the solution process, where averaging is usually not possible.
Indeed, to address this issue, \citet{brown2019deep} create a separate neural network to approximate the average strategy in their Deep CFR model.

Therefore, a natural idea is to design regret-minimization algorithms whose last strategy converges (we call this \emph{last-iterate convergence}), ideally at a faster rate than the average iterate.
Unfortunately, many regret minimization algorithms such as regret matching, regret matching+, and hedge, are known not to satisfy this property empirically and theoretically even for normal-form games.
Although \citet{bowling2015heads} find that in Heads-Up Limit Hold'em poker the last strategy of CFR+ is better than the average strategy, and \citet{farina2019optimistic} observe in some experiments the last-iterate of optimistic OMD and FTRL converge fast, a theoretical understanding of this phenomenon is still absent for EFGs. 

In this work, inspired by recent results on last-iterate convergence in normal-form games~\citep{wei2021linear}, we greatly extend the theoretical understanding of last-iterate convergence of regret-minimization algorithms in two-player zero-sum extensive-form games with perfect recall, and open up many interesting directions both in theory and practice.
First, we show that any optimistic online mirror-descent algorithm instantiated with a strongly convex regularizer \modify{that is continuously differentiable} on the EFG strategy space provably enjoys last-iterate convergence, while CFR with either regret matching or regret matching+ fails to converge.
Moreover, for some of the optimistic algorithms, we further show explicit convergence rates.
In particular, we prove that optimistic mirror descent instantiated with the 1-strongly-convex dilated entropy regularizer~\citep{kroer2020faster}, which we refer to as \emph{Dilated Optimistic Multiplicative Weights Update} (DOMWU), has a \emph{linear} convergence rate under the assumption that there is a unique Nash equilibrium; we note that this assumption was also made by \cite{daskalakis2019last,wei2021linear} in order to achieve similar results for normal-form games. 
\section{Related Work}\label{sec: related}
\paragraph{Extensive-form Games} 
Here we focus on work related to two-player zero-sum perfect-recall games.
Although there are many game-solving techniques such as abstraction \citep{kroer2014extensive,ganzfried2014potential,brown2015hierarchical}, endgame solving \citep{burch2014solving,ganzfried2015endgame}, and subgame solving \citep{moravcik2016refining,brown2019superhuman}, these methods all rely on scalable methods for computing approximate Nash equilibria.
There are several classes of algorithms for computing approximate Nash equilibria, such as double-oracle methods~\citep{mcmahan2003planning}, fictitious play~\citep{brown1951iterative,heinrich2016deep}, first-order methods~\citep{hoda2010smoothing,kroer2020faster}, and CFR methods~\citep{zinkevich2007regret,lanctot2009monte,tammelin2014solving}.
Notably, variants of the CFR approach have achieved significant success in poker games~\citep{bowling2015heads,moravvcik2017deepstack,brown2018superhuman}.
Underlying the first-order and CFR approaches is the sequence-form representation~\citep{von1996efficient}, which allows the problem to be represented as a bilinear saddle-point problem. This leads to algorithms based on smoothing techniques and other first-order methods~\citep{gilpin2008first,kroer2017smoothing,gao2021increasing}, and enables CFR via the theorem connecting no-regret guarantees to Nash equilibrium.

\paragraph{Online Convex Optimization and Optimistic Regret Minimization}
Online convex optimization \citep{zinkevich2003online} is a framework for repeated decision making where the goal is to minimize regret.
When applied to repeated two-player zero-sum games, it is known that the average strategy converges to Nash Equilibria at the rate of $\order(1/\sqrt{T})$ when both players apply regret-minimization algorithms whose regret grows on the order of $\order(\sqrt{T})$~\citep{freund1999adaptive,hart2000simple,zinkevich2007regret}.
Moreover, when the players use optimistic regret-minimization algorithms, the convergence rate is improved to $\order(1/T)$ \citep{rakhlin2013optimization,syrgkanis2015fast}.
Recent works have applied optimism ideas to EFGs, such as optimistic algorithms with dilated regularizers \citep{kroer2020faster,farina2019optimistic}, CFR-like local optimistic algorithms \citep{farina2019stable}, and optimistic CFR algorithms \citep{burch2018time,brown2019solving,farina2020faster}.
However, the theoretical results in all these existing papers consider the average strategy, while we are the first to consider last-iterate convergence in EFGs. 

\paragraph{Last-iterate Convergence in Saddle-point Optimization}
As mentioned previously, two-player zero-sum games can be formulated as saddle-point optimization problems. Saddle-point problems have recently gained a lot of attention due to their applications in machine learning, for example in generative adversarial networks \citep{goodfellow2014generative}.
Basic algorithms, including gradient descent ascent and multiplicative weights update, diverge even in simple instances~\citep{mertikopoulos2018cycles,bailey2018multiplicative}.
In contrast, their optimistic versions, optimistic gradient descent ascent
(OGDA) \citep{daskalakis2018training,mertikopoulos2019optimistic,wei2021linear} and optimistic multiplicative weights update (OMWU) \citep{daskalakis2019last,Lei2020Last,wei2021linear} have been shown to enjoy attractive last-iterate convergence guarantees.
However, almost none of these results apply to the case of EFGs:
\citet{wei2021linear} show a result that implies linear convergence of vanilla OGDA in EFGs (see \pref{cor: vogda rate}), but no results are known for vanilla OMWU or more importantly for algorithms instantiated with \emph{dilated} regularizers which lead to fast iterate updates in EFGs.
In this work we extend the existing results on normal-form games to EFGs, including the practically-important dilated regularizers.
\section{Problem Setup}\label{sec: setup}

We start with some basic notation. 
For a vector $\z$, we use $z_i$ to denote its $i$-th coordinate and $\|\z\|_p$ to denote its $p$-norm (with $\|\z\|$ being a shorthand for $\|\z\|_2$).
For a convex function $\psi$, the associated Bregman divergence is define as
$\breg{\psi}(\bu,\bv)=\psi(\bu)-\psi(\bv)-\inner{\nabla\psi(\bv),\bu-\bv}$,
and $\psi$ is called $\kappa$-strongly convex with respect to the $p$-norm if $\breg{\psi}(\bu,\bv)\ge\frac{\kappa}{2}\|\bu-\bv\|_p^2$ holds for all $\bu$ and $\bv$ in the domain.
The Kullback-Leibler divergence, which is the Bregman divergence with respect to the entropy function, is denoted by $\KLD(\cdot,\cdot)$.
Finally, we use $\Delta_P$ to denote the $(P-1)$-dimensional simplex and $[N]$ to denote the set $\{1,2\dots,N\}$ for some positive integer $N$.

\paragraph{Extensive-form Games as Bilinear Saddle-point Optimization}
 We consider the problem of finding a Nash equilibrium of a two-player zero-sum extensive-form game (EFG) with perfect recall.
 Instead of formally introducing the definition of an EFG \modify{(see \pref{app: treeplex} for an example)}, for the purpose of this work, it suffices to consider an equivalent formulation, which casts the problem as a simple bilinear saddle-point optimization~\citep{von1996efficient}:
\begin{align}\label{eq: minimax form}
	\min_{\x\in\treeX}\max_{y\in\treeY}\x^\top\G\y=\max_{y\in\treeY}\min_{\x\in\treeX}\x^\top\G\y,
\end{align}
where $\G \in [-1,+1]^{M\times N}$ is a known matrix, and $\treeX\subset \R^M$ and $\treeY\subset\R^N$ are two polytopes called \emph{treeplexes} (to be defined soon).
The set of Nash equilibria is then defined as $\trplx^*=\treeX^*\times\treeY^*$, where $\treeX^*=\argmin_{\x\in\treeX}\max_{y\in\treeY}\x^\top\G\y$ and $\treeY^*=\argmax_{y\in\treeY}\min_{\x\in\treeX}\x^\top\G\y$.
Our goal is to find a point $z \in \trplx = \treeX\times\treeY$ that is close to the set of Nash equilibria $\trplx^*$,
and we use the Bregman divergence (of some function $\psi$) between $\z$ and the closest point  in $\trplx^*$ to measure the closeness, that is, $\min_{\z^*\in\trplx^*}\KL(\z^*,\z)$.

For notational convenience, we let $P = M+N$ and $F(\z)=(\G\y,-\G^\top\x)$ for any $\z = (\x, \y) \in \trplx \subset \R^P$.
Without loss of generality, we assume $\|F(\z)\|_\infty\le 1$ for all $\z\in\trplx$ (which can always be ensured by normalizing the entries of $\G$ accordingly).

\paragraph{Treeplexes} 
The structure of the EFG is implicitly captured by the treeplexes $\treeX$ and $\treeY$, which are generalizations of simplexes that capture the sequential structure of an EFG.
The formal definition is as follows.
\modify{
(In \pref{app: treeplex}, we provide more details on the connection between treeplexes and the structure of the EFG, as well as concrete examples of treeplexes for better illustrations.)
}

\begin{definition}[\citet{hoda2010smoothing}]\label{def:treeplex}
	A treeplex is recursively constructed via the following three operations:
	\begin{enumerate}
		\item Every probability simplex is a treeplex. 
		\item Given treeplexes $\trplx_1,\dots,\trplx_K$, the Cartesian product $\trplx_1\times\cdots\times\trplx_K$ is a treeplex.
		\item (Branching) Given treeplexes $\trplx_1\subset\R^M$ and $\trplx_2\subset\R^N$, and any $i\in[M]$,
		\[
		\brch{\trplx_1}{\trplx_2}{i} = \left\{(\bu, u_i\cdot\bv)\in\R^{M+N}:\bu\in\trplx_1,~\bv\in\trplx_2\right\}	
		\] 
		is a treeplex. 
	\end{enumerate}
\end{definition}

By definition, a treeplex is a tree-like structure built with simplexes, which intuitively represents the tree-like decision space of a single player, and an element in the treeplex represents a strategy for the player.
Let $\csimp^{\trplx}$ denote the collection of all the simplexes in treeplex $\trplx$, which following \pref{def:treeplex} can be recursively defined as: $\csimp^{\trplx} = \{\trplx\}$ if $\trplx$ is a simplex; $\csimp^{\trplx} = \bigcup_{k=1}^K \csimp^{\trplx_i}$ if $\trplx$ is a Cartesian product $\trplx_1\times\cdots\times\trplx_K$;
and $\csimp^{\trplx} = \csimp^{\trplx_1} \cup \csimp^{\trplx_2}$ if $\trplx = \brch{\trplx_1}{\trplx_2}{i}$.
In EFG terminology, $\csimp^{\treeX}$ and $\csimp^{\treeY}$ are the collections of \emph{information sets} for player $\x$ and player $\y$ respectively, which are the decision points for the players, at which they select an action within the simplex. 
For any $\infoi \in \csimp^{\trplx}$, we let $\seqactn_{\infoi}$ denote the set of indices belonging to $\infoi$,
and for any $\z\in\trplx$, we let $\z_{\infoi}$ be the slice of $\z$ whose indices are in $\seqactn_\infoi$.
For each index $i$, we also let $\infoi(i)$ be the simplex $i$ falls into, that is, $i \in \seqactn_{\infoi(i)}$.

In \pref{def:treeplex}, the last branching operation naturally introduces the concept of a \emph{parent variable} for each $\infoi \in \csimp^{\trplx}$, which can again be recursively defined as: if $\trplx$ is a simplex, then it has no parent; if $\trplx$ is a Cartesian product $\trplx_1\times\cdots\times\trplx_K$, then the parent of $\infoi\in \csimp^{\trplx}$ is the same as the parent of $\infoi$ in the treeplex $\trplx_k$ that $\infoi$ belongs to (that is, $\infoi \in \csimp^{\trplx_k}$); 
finally, if $\trplx = \left\{(\bu, u_i\cdot\bv):\bu\in\trplx_1,~\bv\in\trplx_2\right\}$, then for all $h \in \csimp^{\trplx_2}$ without a parent, their parent in $\trplx$ is $u_i$, and for all other $h$, their parents remain the same as in $\trplx_1$ or $\trplx_2$.
We denote by $\seq(h)$ the index of the parent variable of $h$, and let it be $0$ if $h$ has no parent.
For convenience, we let $z_0 = 1$ for all $z \in \trplx$ (so that $z_{\seq(h)}$ is always well-defined).
Also define $\csimp_i = \{h \in \csimp^\trplx: \seq(h) = i\}$ to be the collection of simplexes whose parent index is $i$.

Similarly, for an index $i$, its parent index is defined as $p_i = \seq(h(i))$, and $i$ is called a \emph{terminal index} if it is not a parent index (that is, $i \neq p_j$ for all $j$).
Finally, for an element $\z \in \trplx$ and an index $i$, we define $q_i=z_i/z_{p_i}$.
In EFG terminology, $q_i$ specifies the probability of selecting action $i$ in the information set $\infoi(i)$ according to strategy $\z$.

\section{Optimistic Regret-minimization Algorithms}\label{sec:algs}
There are many different algorithms for solving bilinear saddle-point problems over general constrained sets.
We focus specifically on a family of regret-minimization algorithms, called Optimistic Online Mirror Descent (\oomd)~\citep{rakhlin2013optimization}, which are known to be highly efficient.
In contrast to the CFR algorithm and its variants, which minimize a local regret notion at each information set (which upper bounds global regret), the algorithms we consider explicitly minimize global regret.
As our main results in the next section show, these global regret-minimization algorithms enjoy last-iterate convergence, while CFR provably diverges.

Specifically, given a step size $\eta>0$ and a convex function $\psi$ (called a \emph{regularizer}), \oomd sequentially performs the following update for $t=1,2,\dots$,
\begin{align*}
	\x_{t} &= \argmin_{\x\in\treeX} \Big\{\eta \inner{\x,\G \y_{t-1}} + D_\psi(\x,\xp_{t})\Big\}, &\xp_{t+1} &= \argmin_{\x\in\treeX} \Big\{\eta \inner{\x,\G \y_{t}} + D_\psi(\x,\xp_{t})\Big\},\\
	\y_{t} &= \argmin_{\y\in\treeY} \Big\{\eta \inner{\y,-\G^\top \x_{t-1}} + D_\psi(\y,\yp_{t})\Big\}, &\yp_{t+1} &= \argmin_{\y\in\treeY} \Big\{\eta \inner{\y,-\G^\top \x_{t}} + D_\psi(\y,\yp_{t})\Big\},
\end{align*} 
with $(\xp_1, \yp_1) = (\x_0, \y_0) \in \trplx$ being arbitrary.
Using shorthands $\z_{t}=(\x_t,\y_t)$, $\zp_t=(\xp_t,\yp_t)$, $\psi(\z)=\psi(\x)+\psi(\y)$ and recalling the notation $F(\z)=(\G \y,-\G^\top \x)$,
the updates above can be compactly written as \oomd with regularizer $\psi$ over treeplex $\trplx$:
\begin{align}
    \z_{t}= \argmin_{\z\in\calZ} \Big\{\eta \inner{\z, F(\z_{t-1})} + D_\psi(\z,\zp_{t})\Big\},~\,
    \zp_{t+1}= \argmin_{\z\in\calZ} \Big\{\eta \inner{\z,  F(\z_t)} + D_\psi(\z,\zp_t)\Big\}.~\label{eq: oomd update} 
\end{align}

Below, we discuss four different regularizers and their resulting algorithms (throughout, we use notations $\eucf$ for regularizers based on Euclidean norm and $\entf$ for regularizers based on entropy).

\paragraph{Vanilla Optimistic Gradient Descent Ascent (\vogda)}
Define the vanilla squared Euclidean norm regularizer as $\veuc(\z)=\frac{1}{2}\sum_{i}z_i^2$.
We call \oomd instantiated with $\psi = \veuc$ Vanilla Optimistic Gradient Descent Ascent (\vogda).
In this case, the Bregman divergence is $\breg{\veuc}(\z,\z')=\frac{1}{2}\|\z-\z'\|^2$ (by definition $\veuc$ is thus $1$-strongly convex with respect to the $2$-norm), and the updates simply become projected gradient descent.
For \vogda there is no closed-form for \pref{eq: oomd update}, since projection onto the treeplex $\trplx$ is required.
Nevertheless, the solution can still be computed in $\order(P^2\log P)$ time (recall that $P$ is the dimension of $\trplx$)~\citep{gilpin2008first}.

\paragraph{Vanilla Optimistic Multiplicative Weight Update (\vomwu)}
Define the vanilla entropy regularizer as $\vent(\z)=\sum_{i}z_i\ln z_i$.
We call \oomd with $\psi = \vent$ Vanilla Optimistic Multiplicative Weights Update (\vomwu).
The Bregman divergence in this case is the generalized KL divergence: $\breg{\vent}(\z,\z') = \sum_{i}z_i\ln(z_i/z_i')-z_i+z_i'$.
Although it is well-known that $\vent$ is $1$-strongly convex with respect to the $1$-norm for the special case when $\trplx$ is a simplex,
this is not true generally on a treeplex.
Nevertheless, it can still be shown that $\vent$ is 1-strongly convex with respect to the $2$-norm; see \pref{app:sec4}.

The name ``Multiplicative Weights Update'' is inherited from case when $\treeX$ and $\treeY$ are simplexes, in which case the updates in \pref{eq: oomd update} have a simple multiplicative form.
We emphasize, however, that in general \vomwu does not admit a closed-form update.
Instead, to solve \pref{eq: oomd update}, one can equivalently solve a simpler dual optimization problem; see~\citep[Proposition 1]{zimin2013online}.


The two regularizers mentioned above ignore the structure of the treeplex.
\emph{Dilated Regularizers}~\citep{hoda2010smoothing}, on the other hand, take the structure into account and allow one to decompose the update into simpler updates at each information set.
Specifically, given any convex function $\psi$ defined over the simplex and a weight parameter $\balpha \in \R_+^{\csimp^\trplx}$, the dilated version of $\psi$ defined over $\trplx$ is:
\begin{align}
	\psi^{\text{dil}}_{\balpha}(\z)=\sum_{\infoi\in\csimp^\trplx} \alpha_{\infoi} \cdot z_{{\seq(\infoi)}}\cdot\psi\left(\frac{\z_{\infoi}}{z_{\seq(\infoi)}}\right). \label{eq:dilated_reg}
\end{align}
This is well-defined since $\nicefrac{\z_{\infoi}}{z_{\seq(\infoi)}}$ is indeed a distribution within the simplex $\infoi$ (with $q_i$ for $i \in \seqactn_\infoi$ being its entries).
It can also be shown that $\psi^{\text{dil}}_{\balpha}$ is always convex in $\z$~\citep{hoda2010smoothing}.
Intuitively, $\psi^{\text{dil}}_{\balpha}$ applies the base regularizer $\psi$ to the action distribution in each information set and then scales the value by its parent variable and the weight $\alpha_{\infoi}$.
By picking different base regularizers, we obtain the following two algorithms.

\paragraph{Dilated Optimistic Gradient Descent Ascent  (\dogda) \modify{\citep{farina2019optimistic}}}
Define the dilated squared Euclidean norm regularizer $\deuc$ as \pref{eq:dilated_reg} with $\psi$ being the vanilla squared Euclidean norm $\psi(\z) = \frac{1}{2}\sum_{i}z_i^2$.
Direct calculation shows $\deuc(\z) = \frac{1}{2}\sum_{i}\alpha_{\infoi(i)}z_iq_i$.
We call \oomd with regularizer $\deuc$ the Dilated Optimistic Gradient Descent Ascent algorithm (\dogda).
It is known that there exists an $\balpha$ such that $\deuc$ is $1$-strongly convex with respect to the $2$-norm~\citep{farina2019optimistic}.
Importantly, \dogda decomposes the update \pref{eq: oomd update} into simpler gradient descent-style updates at each information set, as shown below.
	\begin{lemma}[\citet{hoda2010smoothing}]
		If $\z'=\argmin_{\z\in\trplx}\left\{\eta\inner{\z,\f}+{\breg{\deuc}(\z,\zp)}\right\}$, then for every $\infoi\in\csimp^\trplx$, the corresponding vector $\q_\infoi' = \frac{\z_{\infoi}'}{z_{\seq(\infoi)}'}$ can be computed by:
	\begin{align}\label{eq: odgda subproblem}
		\q_{\infoi}'=\argmin_{\q_\infoi \in \Delta_{|\seqactn_\infoi|}}\left\{\eta\inner{\q_\infoi,\blp_{\infoi}}+\frac{\alpha_{\infoi}}{2}\|\q_\infoi-\qp_\infoi\|^2\right\},
	\end{align}
	where $\qp_\infoi = \frac{\zp_{\infoi}}{\zp_{\seq(\infoi)}}$, $\blp_{\infoi}$ is the slice of $\blp$ whose entries are in $\seqactn_\infoi$, and $\blp$ is defined through:
	\begin{align*}
		\lp_i=f_i+\sum_{\infoi\in \csimp_i}\left(\inner{\q_{\infoi}',\blp_{\infoi}}+\frac{\alpha_\infoi}{2\eta}\|\q_{\infoi}'-\qp_{\infoi}\|^2\right).
	\end{align*}
	\end{lemma}
While the definitions of $\q_{\infoi}'$ and $\blp$ are seemingly recursive, one can verify that they can in fact be computed in a ``bottom-up'' manner, starting with the terminal indices. 
Although \pref{eq: odgda subproblem} still does not admit closed-form solution, 
it only requires projection onto a simplex, which can be solved efficiently, see e.g.~\citep{condat2016fast}.
Finally, with $\q_{\infoi}'$ computed for all $\infoi$, $\z'$ can be calculated in a ``top-down'' manner by definition.

\paragraph{Dilated Optimistic Multiplicative Weight Update (\domwu) \modify{\citep{kroer2020faster}}}
Finally, define the dilated entropy regularizer $\dent$ as \pref{eq:dilated_reg} with $\psi$ being the vanilla entropy $\psi(\z) = \sum_{i}z_i \ln z_i$.
Direct calculation shows $\dent(\z) = \sum_{i}\alpha_{\infoi(i)}z_i \ln q_i$.
We call \oomd with regularizer $\dent$ the Dilated Optimistic Multiplicative Weights Update algorithm (\domwu).
Similar to \dogda, there exists an $\balpha$ such that $\dent$ is $1$-strongly convex with respect to the $2$-norm~\citep{kroer2020faster}.\footnote{\citet{kroer2020faster} also show a better strong convexity result with respect to the $1$-norm. We focus on the $2$-norm here for consistency with our other results, but our analysis can be applied to the $1$-norm case as well.}
Moreover, in contrast to all the three algorithms mentioned above, the update of \domwu has a closed-form solution:
\begin{lemma}[\citet{hoda2010smoothing}]\label{lem: dilated entropy update}
Suppose $\z'=\argmin_{\z\in\trplx}\left\{\eta\inner{\z,\f}+{\breg{\dent}(\z,\zp)}\right\}$. Similarly to the notation $q_i$, define $q'_i = z'_i / z'_{p_i}$ and $\qpp_i = \zpp_i / \zpp_{p_i}$. Then we have
	\begin{align*}
		q_i' \propto  \qpp_i\exp\left(-\eta \lp_i/\alpha_{\infoi(i)}\right),~\text{where}~\lp_i=f_i-\sum_{\infoi\in \csimp_i}\frac{\alpha_\infoi}{\eta}\ln\left(\sum_{j\in \seqactn_{\infoi}}\qpp_j\exp\left(-\eta \lp_j/\alpha_\infoi\right)\right).
	\end{align*}
\end{lemma}
This lemma again implies that we can compute $q_i'$ bottom-up, and then $\z'$ can be computed top-down.
This is similar to \dogda, except that all updates have a closed-form.

\section{Last-iterate Convergence Results}
In this section, we present our main last-iterate convergence results for the global regret-minimization algorithms discussed in~\pref{sec:algs}.
Before doing so, we point out again that the sequence produced by the well-known CFR algorithm may diverge (even if the average converges to a Nash equilibrium).
Indeed, this can happen even for a simple normal-form game, as formally shown below.
\begin{theorem}\label{thm: cfr diverges}
    In the rock-paper-scissors game, 
    CFR (with some particular initialization) produces a diverging sequence. 
\end{theorem}
In fact, we empirically observe that all of CFR, CFR+~\citep{tammelin2014solving} (with simultaneous updates), and their optimistic versions~\citep{farina2020faster} may diverge in the rock-paper-scissors game.
We introduce the algorithms and show the results in \pref{app:cfr}.

On the contrary, every algorithm from the \oomd family given by \pref{eq: oomd update} ensures last-iterate convergence, as long as the regularizer is strongly convex and continuously differentiable.

\begin{theorem}\label{thm: asy convergence}
    Consider the update rules in \pref{eq: oomd update}.
    Suppose that $\psi$ is $1$-strongly convex with respect to the $2$-norm \modify{and continuously differentiable on the entire domain}, and $\eta\le\frac{1}{8P}$. Then $\z_t$ converges to a Nash equilibrium as $t \rightarrow \infty$.
\end{theorem}
 
\modify{
As mentioned, $\veuc$ and $\vent$ are both $1$-strongly convex with respect to $2$-norm, so are $\deuc$ and $\dent$ under some specific choice of $\balpha$ (in the rest of the paper, we fix this choice of $\balpha$).
However, only $\veuc$ and $\deuc$ are continuously differentiable in the entire domain. 
Therefore, \pref{thm: asy convergence} provides an asymptotic convergence result only for $\vogda$ and $\dogda$, but not $\vomwu$ and $\domwu$.
Nevertheless, below, we resort to different analyses to show a concrete last-iterate convergence rate for three of our algorithms, which is a much more challenging task.
}

First of all, note that \citep[Theorem 5, Theorem 8]{wei2021linear} already provide a general last-iterate convergence rate for \vogda over polytopes.
Since treeplexes are polytopes, we can directly apply their results and obtain the following corollary.
\begin{corollary}\label{cor: vogda rate}
    Define $\dist(\z,\trplx^*) = \min_{\z^* \in \trplx^*}\|\z-\z^*\|^2$. For $\eta\le\frac{1}{8P}$, \vogda guarantees 
    \begin{align*}
        \dist(\z_t,\trplx^*)\le 64\dist(\zp_1,\trplx^*)(1+C_1)^{-t},
    \end{align*}
    where $C_1>0$ is some constant that depends on the game and $\eta$.
\end{corollary}

However, the results for \vomwu in~\citep[Theorem 3]{wei2021linear} is very specific to normal-form game (that is, when $\treeX$ and $\treeY$ are simplexes) and thus cannot be applied here.
Nevertheless, we are able to extend their analysis to get the following result.
\begin{theorem}
    \label{thm: vomwu rate}
    If the EFG has a unique Nash equilibrium $\z^*$, then
    \vomwu with step size $\eta\leq \frac{1}{8P}$ guarantees
    $
             \frac{1}{2}\|\zp_t-\z^*\|^2 \leq \breg{\vent}(\z^*,\zp_t) \leq \frac{C_2}{t},
    $
        where $C_2>0$ is some constant depending on the game, $\zp_1$, and $\eta$. 
\end{theorem}

We note that the uniqueness assumption is often required in the analysis of \omwu even for normal-form games~\citep{daskalakis2019last,wei2021linear} (although \citep[Appendix~A.5]{wei2021linear} provides empirical evidence to show that this may be an artifact of the analysis).
Also note that for normal-form games, \citep[Theorem 3]{wei2021linear} show a linear convergence rate, whereas here we only show a slower sub-linear rate, due to additional complications introduced by treeplexes (see more discussions in the next section). Whether this can be improved is left as a future direction. 

On the other hand, thanks to the closed-form updates of \domwu, 
we are able to show the following linear convergence rate for this algorithm.
\begin{theorem}
    \label{thm: domwu rate}
    If the EFG has a unique Nash equilibrium $\z^*$, then
    \ODMWU with step size $\eta\leq \frac{1}{8P}$ guarantees
    $
             \frac{1}{2}\|\z_t-\z^*\|^2 \leq\breg{\dent}(\z^*,\z_t) \leq C_3(1+C_4)^{-t},
    $
    where $C_3,C_4>0$ are constants that depend on the game, $\zp_1$, and $\eta$. 
\end{theorem}

To the best of our knowledge, this is the first last-iterate convergence result for algorithms with dilated regularizers. 
Unfortunately, due to technical difficulties, we were unable to prove similar results for \dogda (see \pref{app:OMWU_proofs} for more discussion). 
We leave that as an important future direction.


\section{Analysis Overview}
In this section, we provide an overview of our analysis.
It starts from the following standard one-step regret analysis of \oomd (see, for example, \citep[Lemma 1]{wei2021linear}):
\begin{lemma}
	\label{lem: regret bound omwu}
	Consider the update rules in \pref{eq: oomd update}.
	Suppose that $\psi$ is 1-strongly convex with respect to the $2$-norm, $\|F(\z_1)-F(\z_2)\| \le L\|\z_1-\z_2\|$ for all $\z_1,\z_2\in\trplx$ and some $L > 0$, and $\eta\leq \frac{1}{8L}$.
    Then for any $\z\in\calZ$ and any $t\geq 1$, we have
	\begin{align*}
	\eta F(\z_t)^\top (\z_t-\z) \leq D_\psi(\z, \zp_t) - D_\psi(\z, \zp_{t+1}) - D_\psi(\zp_{t+1}, \z_t) - \tfrac{15}{16}D_\psi(\z_t, \zp_t)+ \tfrac{1}{16} D_\psi(\zp_t, \z_{t-1}).
	\end{align*}
\end{lemma}
Note that the Lipschitz condition on $F$ holds in our case with $L = P$ since
\[
    \|F(\z_1)-F(\z_2)\|= \sqrt{\|\G(\y_1-\y_2)\|^2+\|\G^\top(\x_1-\x_2)\|^2}\le \sqrt{P\|\z_1-\z_2\|_1^2}\le P \|\z_1-\z_2\|,
\]
which is also why the step size is chosen to be $\eta \leq \frac{1}{8P}$ in all our results.
In the following, we first prove \pref{thm: asy convergence}.
Then, we review the convergence analysis of~\citep{wei2021linear} for \omwu in normal-form games,
and finally demonstrate how to prove \pref{thm: vomwu rate} and \pref{thm: domwu rate} by building upon this previous work and addressing the additional complications from EFGs.

\subsection{Proof of \pref{thm: asy convergence}}
For any $\z^*\in\trplx^*$, by optimality of $\z^*$ we have: 
\[
    F(\z_t)^\top (\z_t-\z^*)=\x_t^\top\G\y_t-\x_t^\top\G\y_t+\x_t^\top\G\y^*-{\x^*}^\top\G\y_t\ge {\x^*}^\top\G\y^*-{\x^*}^\top\G\y^*= 0.
\]
Thus, taking $\z=\z^*$ in \pref{lem: regret bound omwu} and rearranging, we arrive at
\begin{align*}
     \KL(\z^*, \zp_{t+1})\leq \KL(\z^*, \zp_t) - \KL(\zp_{t+1}, \z_{t}) - \tfrac{15}{16}\KL(\z_t, \zp_t) + \tfrac{1}{16}\KL(\zp_t, \z_{t-1}). 
\end{align*}
Defining $\Theta_t=\KL(\z^*, \zp_t) + \tfrac{1}{16}\KL(\zp_t, \z_{t-1})$ and $\zeta_t=\KL(\zp_{t+1}, \z_{t}) +\KL(\z_t, \zp_t)$, we rewrite the inequality above as
\begin{align} \label{eq: simple recursion}
    \Theta_{t+1} \leq \Theta_t - \tfrac{15}{16}\zeta_t.   
\end{align}
We remark that similar inequalities appear in \citep[Eq. (3) and Eq. (4)]{wei2021linear}, but here we pick $\z^*\in\trplx^*$ arbitrarily while they have to pick a particular $\z^*\in\trplx^*$ (such as the projection of $\zp_t$ onto $\trplx^*$).
Summing \pref{eq: simple recursion} over $t$, telescoping, and applying the strong convexity of $\psi$, we have 
\begin{align*}
    \Theta_1\ge \Theta_1-\Theta_T\ge\frac{15}{16}\sum^{T-1}_{t=1}\zeta_t\ge\frac{15}{32}\sum^{T-1}_{t=1}\|\zp_{t+1}- \z_{t}\|^2 +\|\z_t- \zp_t\|^2\ge\frac{15}{64}\sum^{T-1}_{t=2}\|\z_{t}-\z_{t-1}\|^2.
\end{align*}
\modify{
Similar to the last inequality, we also have $\Theta_1\ge  \frac{15}{64}\sum^{T-1}_{t=1}\|\zp_{t+1}-\zp_{t}\|^2$ since $2\|\zp_{t+1}- \z_{t}\|^2 +2\|\z_t- \zp_t\|^2\ge \|\zp_{t+1}-\zp_{t}\|^2$.
Therefore, we conclude that $\|\z_t- \zp_t\|$, $\|\z_{t+1}- \z_{t}\|$, and $\|\zp_{t+1}- \zp_{t}\|$ all converge to $0$ as $t\to\infty$.
On the other hand, since the sequence $\{\z_1, \z_2, \ldots,\}$ is bounded, by the Bolzano-Weierstrass theorem, there exists a convergent subsequence, which we denote 
by $\{\z_{i_1}, \z_{i_2}, \ldots,\}$. 
Let $\z_{\infty}=\lim_{\tau \to \infty}\z_{i_{\tau}}$.
By 
$\|\zp_{t}-\z_{t}\|\to 0$ we also have $\z_{\infty}=\lim_{\tau \to \infty}\zp_{i_{\tau}}$.
Now, using the first-order optimality condition of $\zp_{t+1}$, we have for every $\z'\in\trplx$,
\[
(\nabla\psi(\zp_{t+1}) - \nabla\psi(\zp_t) + \eta F(\z_t))^\top (\z'-\zp_{t+1}) \geq 0.
\]
Apply this with $t=i_{\tau}$ for every $\tau$ and let $\tau\to \infty$, we obtain
\begin{align*}
    0\le&\lim_{\tau \to \infty}(\nabla\psi(\zp_{i_{\tau}+1}) - \nabla\psi(\zp_{i_{\tau}}) + \eta F(\z_{i_{\tau}}))^\top (\z'-\zp_{i_{\tau}+1})\tag{by the first-order optimality}\\
    =&\lim_{\tau \to \infty}\eta F(\z_{i_{\tau}})^\top (\z'-\zp_{i_{\tau}+1})\tag{by $\|\zp_{t+1}- \zp_{t}\|\to 0$ and the continuity of $\nabla\psi$}\\
    =&~\eta F(\z_{\infty})^\top (\z'-\z_{\infty})\tag{by $\z_{\infty}=\lim_{\tau \to \infty}\z_{i_{\tau}}=\lim_{\tau \to \infty}\zp_{i_{\tau}}$}
\end{align*}
This implies that $\z_\infty$ is a Nash equilibrium.
Finally,
choosing $\z^*=\z_\infty$ in the definition of $\Theta_{t}$, 
we have $\lim_{\tau\to\infty}\Theta_{i_\tau}=0$ because $\lim_{\tau\to\infty}\breg{\psi}(\z_\infty,\zp_{{i_\tau}})=0$ and $\lim_{\tau\to\infty}\|\zp_{{i_\tau}}-\z_{{i_\tau}-1}\|=0$.
Additionally, by \pref{eq: simple recursion} we also have that $\lim_{t\to\infty}\Theta_{t}=0$ as $\Theta_{t}$ is non-increasing.
Therefore, we conclude that the entire sequence $\{\z_1, \z_2, \ldots\}$ converges to $\z_\infty$.
}
On the other hand, since \oomd is a regret-minimization algorithm, it is well known that the average iterate converges to a Nash equilibrium~\citep{freund1999adaptive}. 
Consequently, combining the two facts above implies that $\z_t$ has to converge to a Nash equilibrium, which proves \pref{thm: asy convergence}. 

We remark that \pref{lem: regret bound omwu} holds for general closed convex domains as shown in \citep{wei2021linear}.
Consequently, with the same argument, \pref{thm: asy convergence} holds more generally as long as $\treeX$ and $\treeY$ are closed convex sets. 
While the argument is straightforward, we are not aware of similar results in prior works.
Also note that unlike \pref{thm: vomwu rate} and \pref{thm: domwu rate}, \pref{thm: asy convergence} holds without the uniqueness assumption for \vomwu and \domwu.

\subsection{Review for normal-form games}
To better explain our analysis and highlight its novelty, we first review the two-stage analysis of \citep{wei2021linear} for \omwu in normal-form games, a special case of our setting when $\treeX$ and $\treeY$ are simplexes.
Note that both \vomwu and \domwu reduce to \omwu in this case. 
As with \pref{thm: vomwu rate} and \pref{thm: domwu rate}, the normal-form OMWU results assume a unique Nash equilibrium $\z^*$.
With this uniqueness assumption and \citep[Lemma C.4]{mertikopoulos2018cycles}, \citet{wei2021linear} show the following inequality 
\begin{align}\label{eq: zeta ge l2 norm}
    \zeta_t = \KL(\zp_{t+1}, \z_{t}) +\KL(\z_t, \zp_t) \ge C_5\|\z^*-\zp_{t+1}\|^2
\end{align}
for some problem-dependent constant $C_5>0$,
which, when combined with \pref{eq: simple recursion}, implies that if the algorithm's current iterate is far from $\z^*$, then the decrease in $\Theta_t$ is more substantial, that is, the algorithm makes more progress on approaching $\z^*$.
To establish a recursion, however, we need to connect the $2$-norm back to the Bregman divergence (a reverse direction of strong convexity).
To do so, \citet{wei2021linear} argue that $\zp_{t+1,i}$ can be lower bounded by another problem-dependent constant for $i\in\supp(\z^*)$~\citep[Lemma 19]{wei2021linear}, where $\supp(\z^*)$ denotes the support of $\z^*$.
This further allows them to lower bound $\|\z^*-\zp_{t+1}\|$ in terms of $\breg{\psi}(\z^*,\zp_{t+1})$ (which is just $\KLD(\z^*,\zp_{t+1})$), 
leading to
\begin{align}
    \zeta_{t}=\KL(\zp_{t+1}, \z_{t}) +\KL(\z_t, \zp_t) \geq C_6\KL(\z^*,\zp_{t+1})^2, \label{eq: bregman a}
\end{align}
for some $C_6>0$.
On the other hand, ignoring the nonnegative term $\KL(\z_t, \zp_t)$, we also have:
\begin{align}
    \zeta_{t}=\KL(\zp_{t+1}, \z_{t}) +\KL(\z_t, \zp_t) \geq \KL(\zp_{t+1}, \z_{t}) \ge \frac{1}{4}\KL(\zp_{t+1}, \z_{t})^2, \label{eq: bregman b}
\end{align}
where the last step uses the fact that $\zp_{t+1}$ and $\z_{t}$ are close~\citep[Lemma 17 and Lemma 18]{wei2021linear}.
Now, \pref{eq: bregman a} and \pref{eq: bregman b} together imply
$
    6\zeta_t\ge {2C_6\KL(\z^*,\zp_{t+1})^2}+{\KL(\zp_{t+1}, \z_{t})^2}\ge\min\left\{{C_6},\frac{1}{2}\right\}\Theta_{t+1}^2.
$
Plugging this back into \pref{eq: simple recursion}, we obtain a recursion
\begin{align}\label{eq:lemma 2 eq 1}
    \Theta_{t+1} \leq \Theta_t -C_7\Theta_{t+1}^2
\end{align}
for some $C_7>0$,  which then implies $\Theta_t=O(1/t)$~\citep[Lemma 12]{wei2021linear}.
This can be seen as the first and slower stage of the convergence behavior of the algorithm.

To further show a linear convergence rate, they argue that there exists a constant $C_8>0$ such that when the algorithm's iterate is reasonably close to $\z^*$ in the following sense:
\begin{align}\label{eq: close z zp}
    \max\{\|\z^*-\zp_t\|_1, \|\z^*-\z_t\|_1\} \leq C_8,
\end{align}
the following improved version of \pref{eq: bregman a} holds (note the lack of square on the right-hand side):
\begin{align}
    \zeta_t =\KL(\zp_{t+1}, \z_{t}) +\KL(\z_t, \zp_t)\geq  C_9\KL(\z^*,\zp_{t+1})\label{eq: bregman 2}
\end{align}
for some constant $0<C_9<1$. 
Therefore, using the $1/t$ convergence rate derived in the first stage, there exists a $T_0$ such that when $t\ge T_0$, \pref{eq: close z zp} holds and the algorithm enters the second stage.
In this stage, combining \pref{eq: bregman 2} and the fact $\zeta_t \geq \KL(\zp_{t+1}, \z_{t})$ gives $\zeta_t \geq \frac{C_9}{2}\Theta_{t+1}$,
which, together with \pref{eq: simple recursion} again,
implies an improved recursion $\Theta_{t+1}\le \Theta_{t}-\frac{15}{32}C_9\Theta_{t+1}$.
This finally shows a linear convergence rate $\Theta_t=O((1+\rho)^{-t})$ for some problem-dependent constant $\rho>0$.

\subsection{Analysis of \pref{thm: vomwu rate} and \pref{thm: domwu rate}}\label{sec:analysis of thm 6 7}


While we mainly follow the steps of the analysis of~\citep{wei2021linear} discussed above to prove \pref{thm: vomwu rate} and \pref{thm: domwu rate}, we remark that the generalization is highly non-trivial.
First of all, we have to prove \pref{eq: zeta ge l2 norm} for $\trplx$ being a general treeplex, which does not follow \citep[Lemma C.4]{mertikopoulos2018cycles} since its proof is very specific to simplexes.
Instead, we prove it by writing down the primal-dual linear program of \pref{eq: minimax form} and applying the strict complementary slackness; see 
\pref{app:slackness} for details. 

Next, to connect the $2$-norm back to the Bregman divergence (which is not the simple KL divergence anymore, especially for \domwu), we prove the following for \vomwu:
\begin{align}\label{eq: vomwu pinsker}
    \breg{\psi}(\z^*,\zp_{t+1})&\le\sum_{i\in\supp(\z^*)}\frac{(z^*_i-\zpp_{t+1,i})^2}{\zpp_{t+1,i}}+\sum_{i\notin\supp(\z^*)}\zpp_{t+1,i}\le \frac{3P\|\z^*-\zp_{t+1}\|}{\min_{i\in\supp(\z^*)}\zpp_{t+1,i}},
\end{align}
and the following for \domwu:
\begin{align}\label{eq: domwu pinsker}
    \frac{\KL(\z^*,\zp_{t+1})}{C'}\le \sum_{i\in\supp(\z^*)}\frac{\left({z^*_i-\zpp_{t+1,i}}\right)^2}{z_{i}^*\qpp_{t+1,i}} +\sum_{i\notin\supp(\z^*)}z_{p_i }^*\qpp_{t+1,i}\le \frac{\|\z^*-\zp_{t+1}\|_1}{\min_{i\in \supp(\z^*)}z_i^*\zpp_{t+1,i}},
\end{align}
where $C'=4P\|\balpha\|_\infty$ (see \pref{app:pinsker}).
We then show a lower bound on $z_{t+1,i}$ and $\zpp_{t+1,i}$ for all $i\in\supp(z^*)$, using similar arguments of~\citep{wei2021linear} (see \pref{app:lower bounds}).
Combining \pref{eq: vomwu pinsker} and \pref{eq: domwu pinsker} with \pref{eq: zeta ge l2 norm}, 
we have the counterpart of \pref{eq: bregman a} for both \vomwu and \domwu.

Showing \pref{eq: bregman b} also involves extra complication if we follow their analysis, especially for \vomwu which does not admit a closed-form update.
Instead, we find a simple workaround: by applying \pref{eq: simple recursion} repeatedly, we get
$
    \KL(\z^*,\zp_{1})=\Theta_{1}\ge \cdots\ge \Theta_{t+1}\ge\frac{1}{16}\KL(\zp_{t+1}, \z_{t}), 
$
thus, $\zeta_t\ge\KL(\zp_{t+1}, \z_{t})\ge C_{10}\KL(\zp_{t+1}, \z_{t})^2$ for some $C_{10}>0$ depending on $\KL(\z^*,\zp_{1})$.
Combining this with \pref{eq: bregman a}, and applying them to \pref{eq: simple recursion}, we obtain the recursion
    $\Theta_{t+1}\le \Theta_{t}-C_{11}\Theta_{t+1}^2$ for some $C_{11} >0$ similar to \pref{eq:lemma 2 eq 1},
which implies $\Theta_t=O(1/t)$ for both \vomwu and \domwu and proves \pref{thm: vomwu rate}.

Finally, to show a linear convergence rate, we need to show the counterpart of \pref{eq: bregman 2}, which is again more involved compared to the normal-form game case.
Indeed, we are only able to do so for \domwu by making use of its closed-form update described in \pref{lem: dilated entropy update}.	
Specifically, observe that in \pref{eq: domwu pinsker},
the term $\sum_{i\notin\supp(\z^*)}z_{p_i }^*\qpp_{t+1,i}$ is the one that prevents us from bounding $\KL(\z^*,\zp_{t+1})$ by $\order(\|\z^*-\zp_{t+1}\|^2)$.
Thus, our high-level idea is to argue that $\sum_{i\notin\supp(\z^*)}z_{p_i }^*\qpp_{t+1,i}$ decreases significantly as $\zp_{t}$ gets close enough to $\z^*$.
To do so, we use a bottom-up induction to prove that, for any information set $h\in\csimp^\trplx$, indices $i,j\in\seqactn_h$ such that $i\notin\supp(\z^*)$ and $j\in\supp(\z^*)$, $\widehat{L}_{t,i}$ is significantly larger than $\widehat{L}_{t,j}$ when $\zp_t$ is close to $\z^*$, where $\widehat{\bm{L}}_{t}$ is the counterpart of $\bm{L}$ in \pref{lem: dilated entropy update} when computing of $\qp_{t+1}$.
This makes sure that the term $\sum_{i\notin\supp(\z^*)}z_{p_i }^*\qpp_{t+1,i}$ is dominated by the other term involving $i\in\supp(\z^*)$ in \pref{eq: domwu pinsker},
which eventually helps us show \pref{eq: bregman 2} and the final linear convergence rate in \pref{thm: domwu rate}. 
See \pref{app:thm7} for details.


\section{Experiments}\label{sec:experiments}
\begin{figure}
    \includegraphics[width=0.33\textwidth]{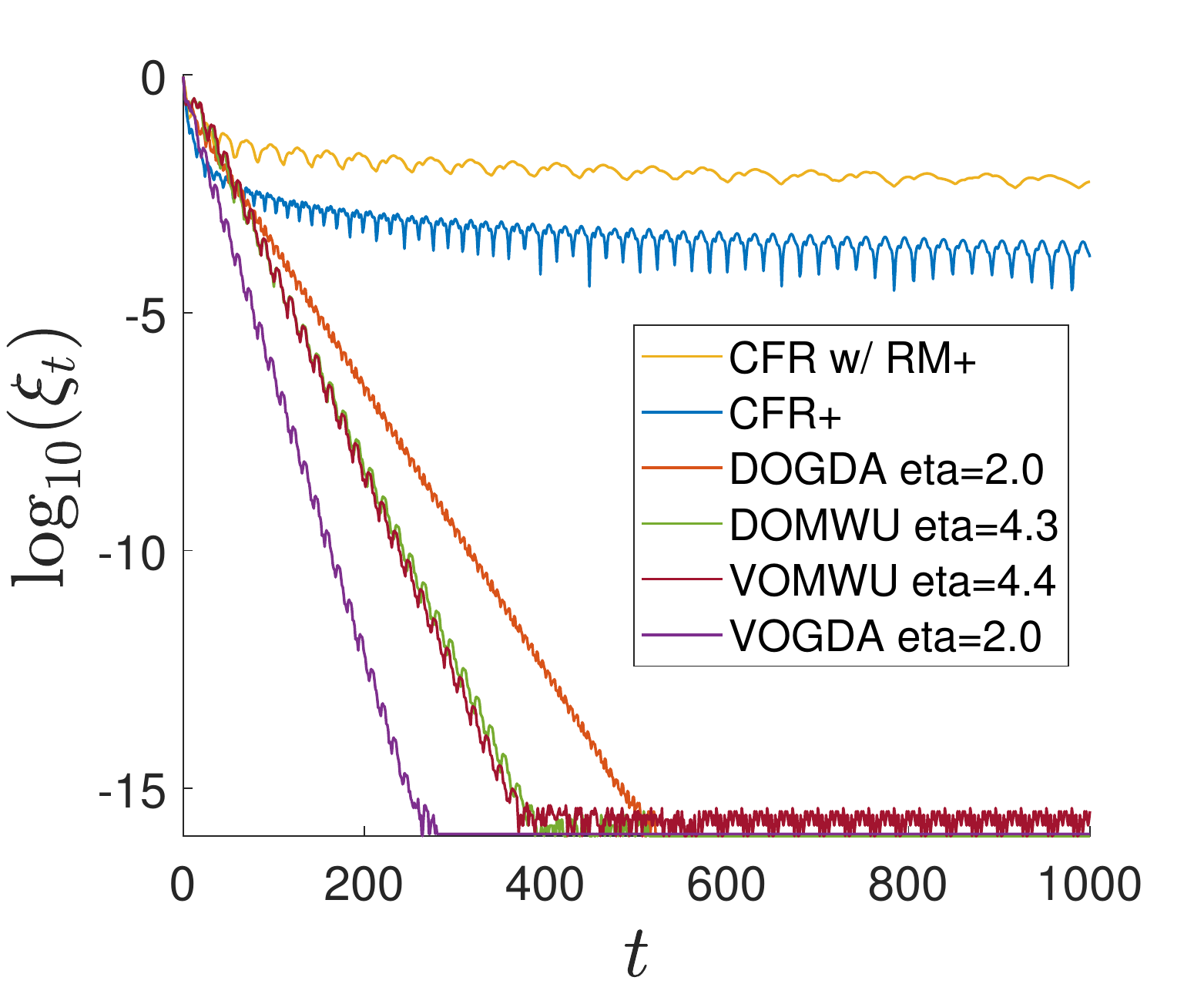}
    \includegraphics[width=0.33\textwidth]{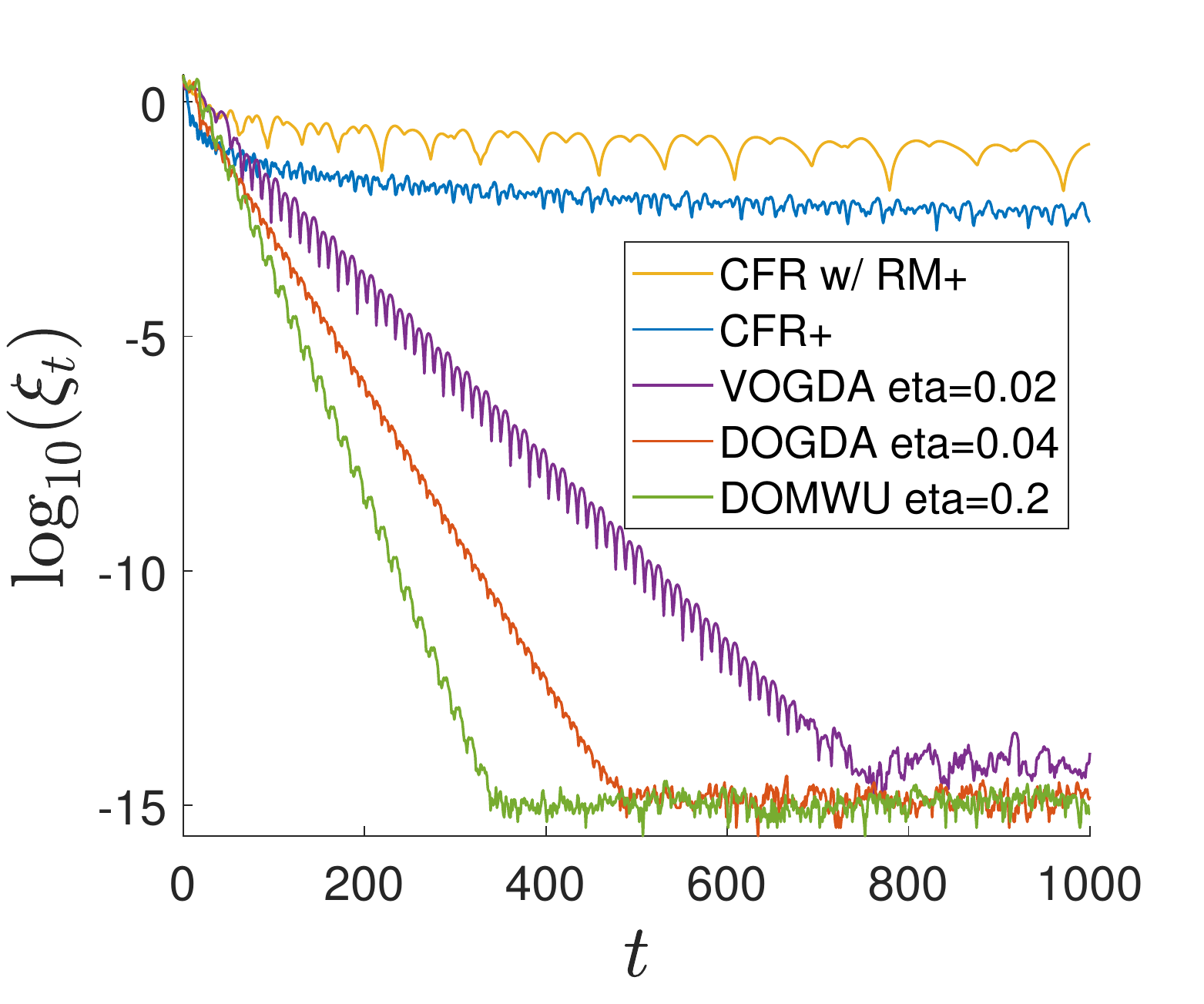}
    \includegraphics[width=0.33\textwidth]{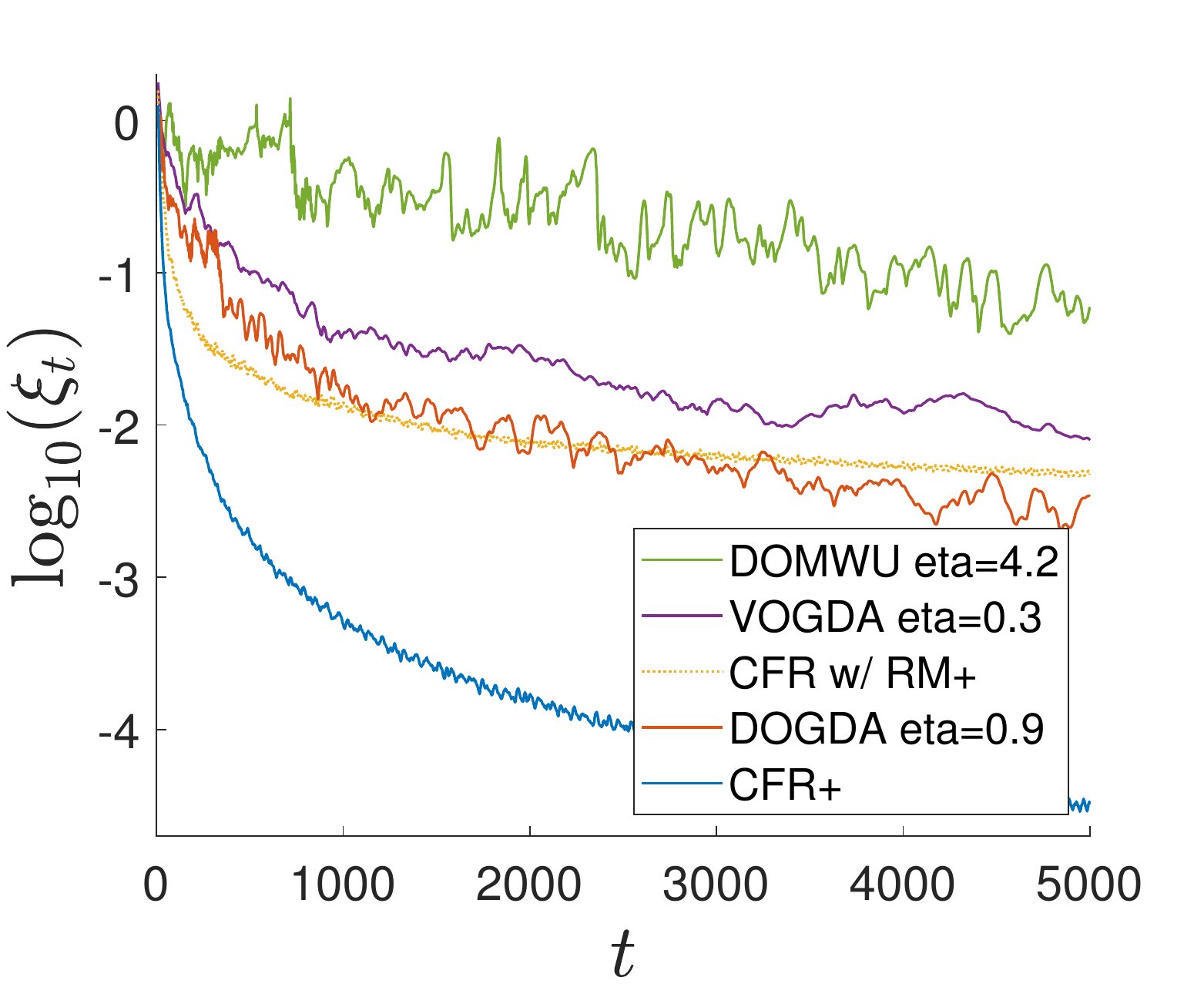}
    \caption{Experiments on Kuhn poker (left), Pursuit-evasion (middle), and Leduc poker (right).
    A description of each game is given in \pref{app:app-experiments}.
    $\xi_t=\max_{\y}\overline{\x}_t^{\top}\G\y-\min_{\x}\x^\top\G\overline{\y}_t$ is the duality gap at time step $t$, where $(\overline{\x}_t,\overline{\y}_t)$ is the approximate Nash equilibrium computed by the algorithm at time $t$ (for the optimistic algorithms, $(\overline{\x}_t,\overline{\y}_t)$ is $({\x}_t,{\y}_t)$ while for the CFR-based algorithms, $(\overline{\x}_t,\overline{\y}_t)$ is the linear average). The legend order reflects the curve order at the right-most point. Due to much higher computation overhead than all the other algorithms, we only run \vomwu on Kuhn poker, the game with the smallest size among the three games.
    For each optimistic algorithm, we fine-tune step size $\eta$ to get better convergence results and show its value in the legends.
    There is no hyperparameter for the CFR-based algorithms.
    All the experiments are run on CPU in a personal computer and the total computation time is less than an hour.
    There is no random seed and the results are all deterministic.
    }\label{fig:fig1} %
\end{figure}
In this section, we experimentally evaluate the algorithms on three standard EFG benchmarks: Kuhn poker \citep{kuhn1950simplified}, Pursuit-evasion \citep{kroer2018robust}, and Leduc poker \citep{southey2005bayes}.
The results are shown in \pref{fig:fig1}.
Besides the optimistic algorithms, we also show two CFR-based algorithms as reference points.
``CFR+'' refers to CFR with alternating updates, linear averaging \citep{tammelin2014solving}, and regret matching+ as the regret minimizer. 
``CFR w/ RM+'' is CFR with regret matching+ and linear averaging. 
We provide the formal descriptions of these two algorithms in \pref{app:cfr} for completeness. 
For the optimistic algorithms, we plot the last iterate performance.
For the CFR-based algorithms, we plot the performance of the linear average of iterates \modify{(recall that the last iterate of CFR-based algorithms is not guaranteed to converge to a Nash equilibrium)}.

For Kuhn poker and Pursuit-evasion (on the left and in the middle of \pref{fig:fig1}), all of the optimistic algorithms perform much better than CFR+, and their curves are nearly straight, showing their linear last-iterate convergence on these games. 

For Leduc poker, although CFR+ performs the best, we can still observe the last-iterate convergence trends of the optimistic algorithms. 
We remark that although \vogda and \domwu have linear convergence rate in theory, the experiment on Leduc uses a step size $\eta$ which is much larger than \pref{cor: vogda rate} and \pref{thm: domwu rate} suggest, which may void the linear convergence guarantee.
This is done because the theoretical step size takes too many iterations before it starts to show improvement.
It is worth noting that CFR+ improves significantly when changing simultaneous updates (that is, CFR w/ RM+) to alternating updates.
Analyzing alternation and combining it with optimistic algorithms is a promising direction.
We provide a description of each game, more discussions, and details of the experiments in \pref{app:app-experiments}. 

\section{Conclusions}
In this work, we developed the first general last-iterate convergence results for solving EFGs.
\modify{
Our paper opens up many potential future directions. 
The recent dilatable global entropy regularizer of \citet{farina2021better} can likely be analyzed using techniques similar to our analysis of VOMWU and DOMWU, and it would likely lead to a linear rate as with DOMWU, due to its closed-form DOMWU-style update.
Other natural questions include whether it is possible to obtain better convergence rates for \vomwu and \dogda, whether one can remove the uniqueness assumption for \vomwu and \domwu,
}
and finally whether it is possible to obtain last-iterate convergence rates for CFR-like optimistic algorithms such as those in \citep{farina2020faster}.
On the practical side, optimistic algorithms with last-iterate convergence guarantees may allow more efficient computation and better incorporation with deep learning-based game-solving approaches.

\begin{ack}
CWL and HL are supported by NSF Award IIS-1943607. 
\end{ack}

\bibliography{ref}
\bibliographystyle{plainnat}

\newpage

\appendix
\modify{
\section{Examples of EFG and Treeplexes}\label{app: treeplex}
In this section, we introduce Kuhn poker \citep{kuhn1950simplified}, a simple EFG as an example to introduce treeplexes and the corresponding definitions. 
In this game, there are three cards in the deck: King, Queen, and Jack.
Both player $\x$ and player $\y$  are dealt one card, while the third card is put aside unseen.
In the first round,  player $\x$ can bet or check. Then, if player $\x$ bets player $\y$ can choose to call or fold. 
If player $\x$ checks then player $\y$ can bet or check.
Finally, if player $\x$ checks and player $\y$ bets, then player $\x$ has a final round where they can call or fold.
If neither player folded, then the player with the higher card wins the pot.
If a player folded then the other player wins the pot.

We show a game tree of Kuhn poker in \pref{fig:efg}.
Players' imperfect information is modeled by information sets.
In \pref{fig:efg}, nodes with the same color belong to the same information set. 
A player cannot distinguish among nodes in a given information set (that belongs to this player). 
For example, player $\x$ cannot distinguish among the blue nodes, since in both nodes, player $\x$ was dealt Queen but does not know whether player $\y$ was dealt Jack or King.  

We can further separate the decision spaces and consider them individually on a per-player basis, which is where the concept of a treeplex arises.
We show player $\x$'s decision space in \pref{fig:tree x}.
Here each circular node is an information set, which is a decision node for player $\x$ where they choose an action.
For example, the blue node $\h_3$ corresponds to the initial state where player $\x$ is dealt Queen, and player $\x$ can choose to bet or check at this node.
Each square node is an observation node, where player $\x$ does not make a decision, but the environment or player $\y$ makes decisions which determine the next decision node for player $\x$.
Each triangular node is a terminal node, where the game ends.

Each index of treeplex $\treeX$ corresponds to a solid, directed, edge in the figure.
In other words, each index corresponds to an action in finite action set $\seqactn_h$ for every $h$ in $\infoX$, the set of information sets that belongs to player $\x$.
Indices (solid edges) are labeled from $x_1$ to $x_{12}$.
More specifically, $\x\in\treeX$ if $\x$ satisfies $x_i\ge 0$ for every index $i$ and for every $ \infoi\in\infoX$,
\begin{align*}
    \sum_{i\in\seqactn_\infoi}x_{i}=x_{\seq(\infoi)},
\end{align*}

where index $\seq(\infoi)\in\Omega_{h'}$ is the unique action such that $\infoi$ can be reached immediately by taking $\seq(\infoi)$ when player $\x$ is in information set $h'$.
When no such action exists, that is, $\infoi$ can be reached immediately in the beginning, we set $\seq(\infoi)=0$ and $x_0=1$.
For example, we must have $x_6 = x_7+x_8$ and $x_5+x_6=x_0=1$ if $\x\in\treeX$.
Intuitively, $x_i$ is the probability taking action $i$, given that the sequential decisions from the environment and player $\y$ can lead to the information set where action $i$ is.
Similarly, we show player $\y$'s decision space in \pref{fig:tree y}, which illustrates treeplex $\treeY$.

\begin{figure}
    \centering
    \includegraphics[width=0.6\textwidth]{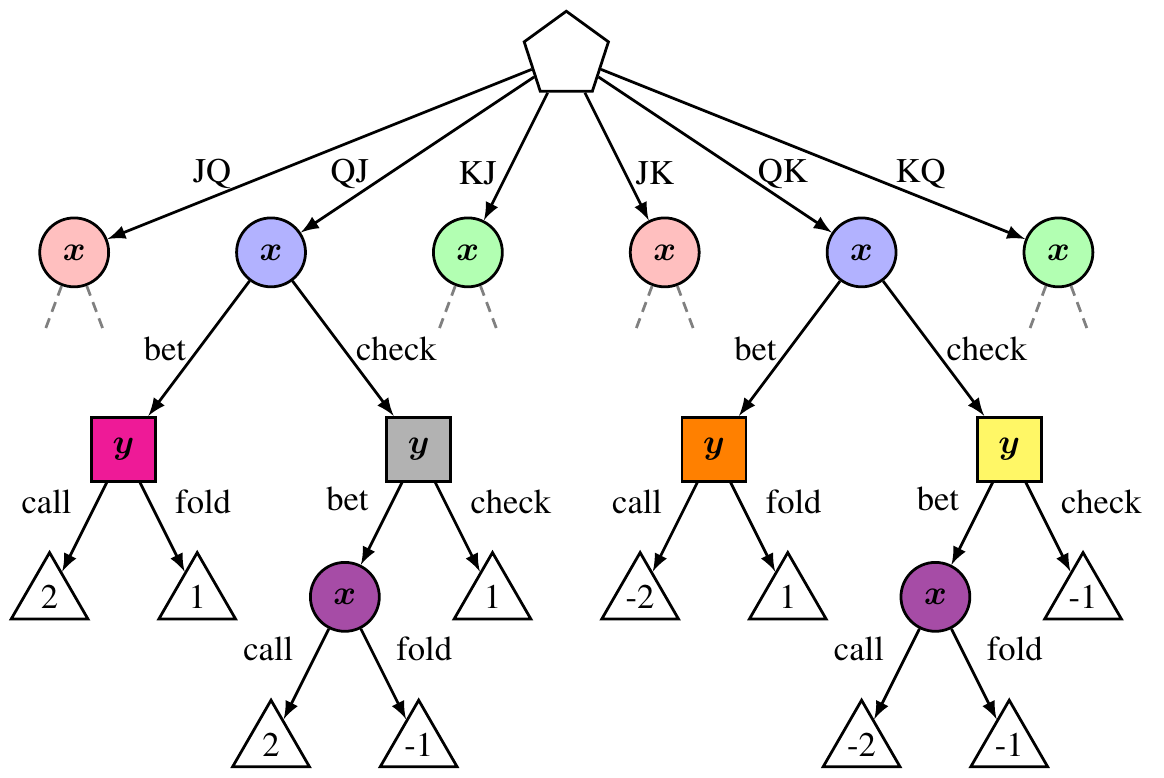}
    \caption{A game tree for Kuhn poker. The edge labeled with JQ means that player $\x$ is dealt Jack and player $\y$ is dealt Queen. The cases are similar at other edges. We omit branches stemming from the green and red nodes, which are similar to what we present for the blue nodes. Circular nodes are player $\x$'s decision nodes, while square nodes are player $\y$'s decision nodes. Triangular nodes are terminal nodes, where the values denote the utility for player $\x$ (and thus the loss for player $\y$). Nodes with the same color belong to the same information set, and a player cannot distinguish among nodes within the same information set, that is, they only know they are at one of these nodes but do not know which node they are at exactly.}
    \label{fig:efg}
\end{figure}
\begin{figure}
    \centering
    \includegraphics[width=0.8\textwidth]{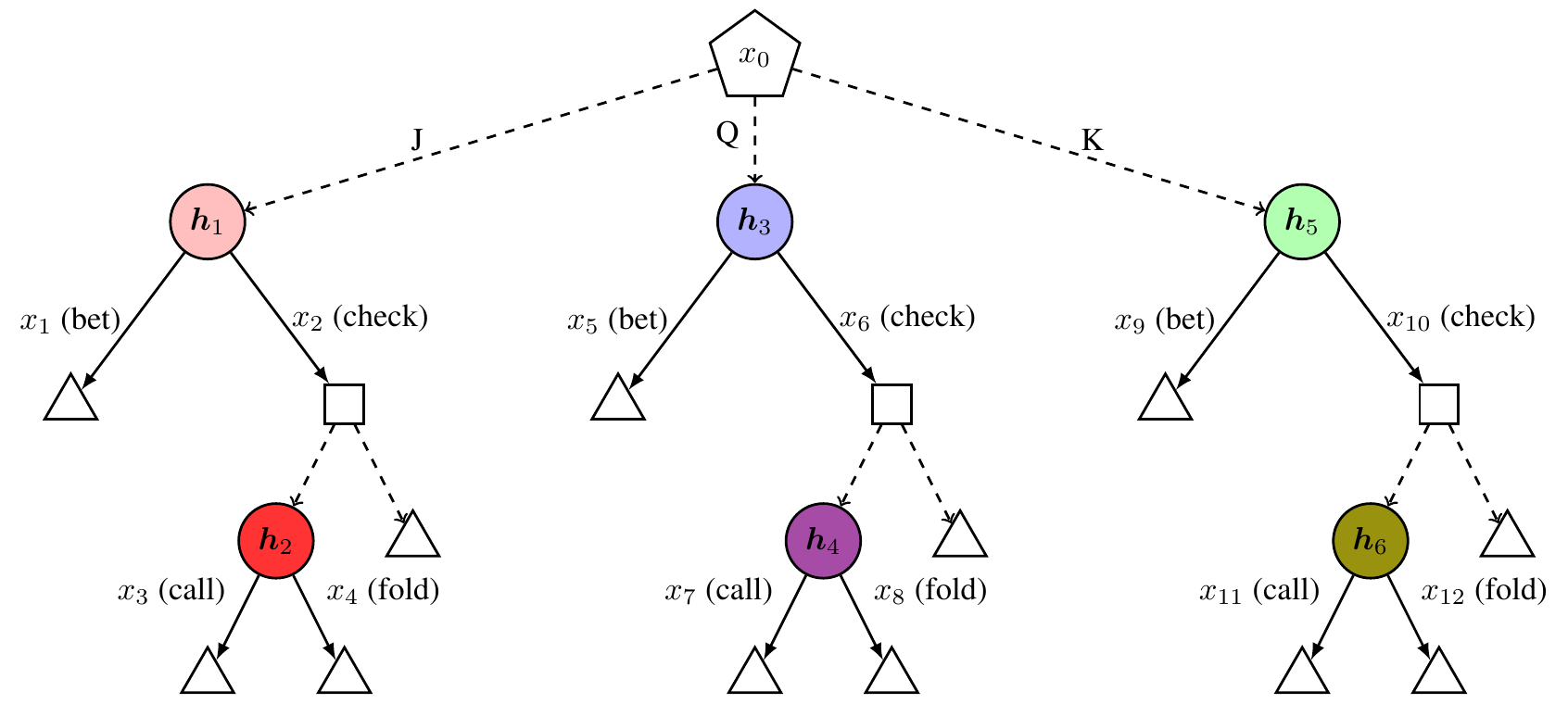}
    \caption{The decision space for player $\x$ and treeplex $\treeX$. Each circular node is an information set, each square node is an observation node, and each triangular node is a termination node, where the decision process ends.
    Each solid edge corresponds an action and one of the indexes in  treeplex $\treeX$. More specifically, we have $M=13$,  $\csimp^{\treeX}=\{\h_1,\dots,\h_6\}$, $\seqactn_{\h_i}=\{2i-1,2i\}$ for $i=1,\dots,6$, $\seq(\h_1)=\seq(\h_3)=\seq(\h_5)=0$, $\seq(\h_2)=2$, $\seq(\h_4)=6$, $\seq(\h_6)=10$.}
    \label{fig:tree x}
\end{figure}
\begin{figure}
    \centering
    \makebox[\textwidth]{\includegraphics[width=1.2\textwidth]{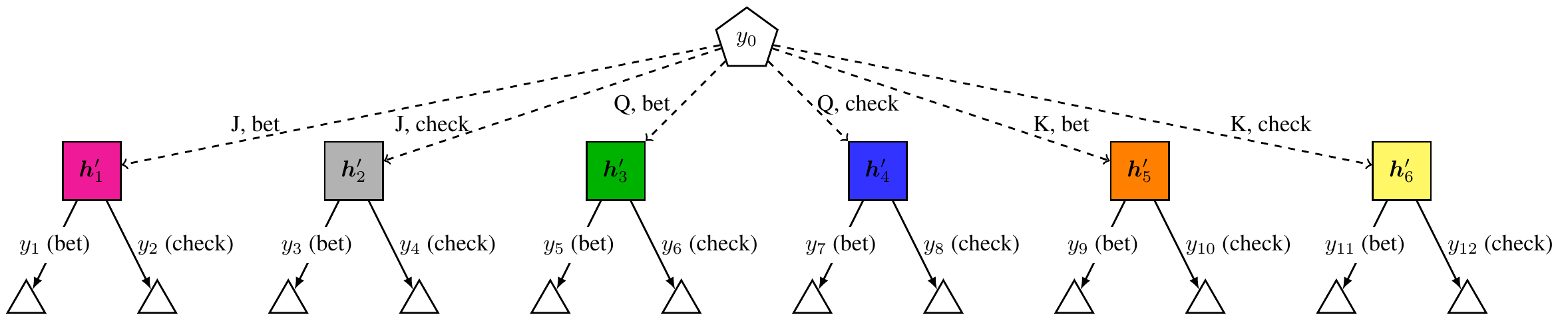}}
    \caption{The decision space for player $\y$ and treeplex $\treeY$. Each square node is an information set and each triangular node is a termination node.
    More specifically, we have $N=13$,  $\csimp^{\treeY}=\{\h_1',\dots,\h_6'\}$, $\seqactn_{\h_i'}=\{2i-1,2i\}$ and $\seq(\h_i')=0$ for $i=1,\dots,6$.}
    \label{fig:tree y}
\end{figure}
}

\section{Omitted Details of  \pref{sec:experiments}}\label{app:app-experiments}

In this section, we provide more details about the experiments. 
\modify{

\subsection{Additional Experiments}
As we mentioned in \pref{sec:experiments}, although \vogda and \domwu have linear convergence rate in theory, we use much larger step sizes $\eta$ in the Leduc poker experiment than what \pref{cor: vogda rate} and \pref{thm: domwu rate} suggest, which explains why we were not able to observe the linear convergence.
Here, we rerun this experiment with a smaller step size for VOGDA and DOMWU. With more iterations, on the order of $10^5$, we observe again that they exhibit fast convergence, as shown in \pref{fig:fig2}. 
\begin{figure}
    \centering
    \includegraphics[width=0.45\textwidth]{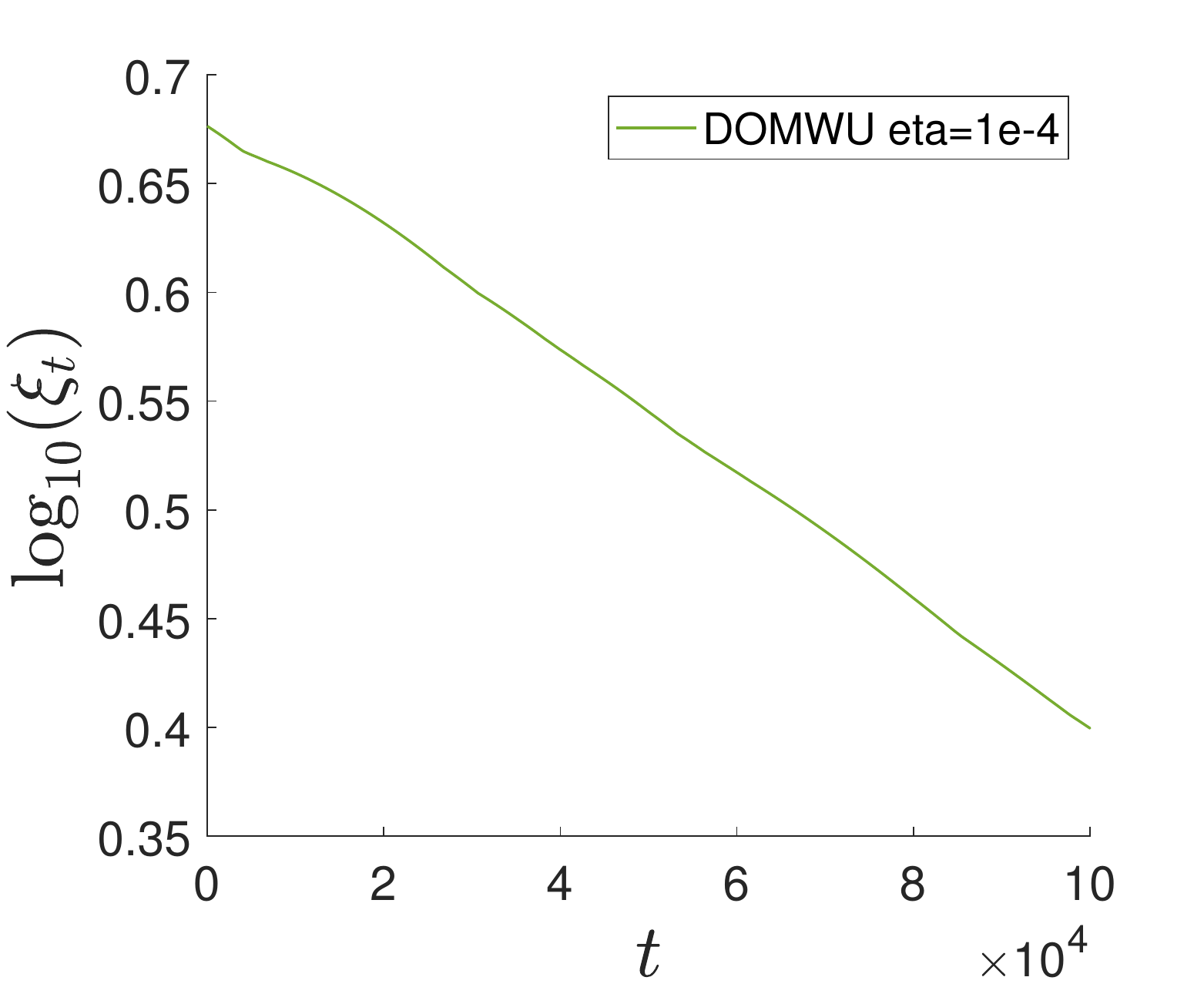}
    \includegraphics[width=0.45\textwidth]{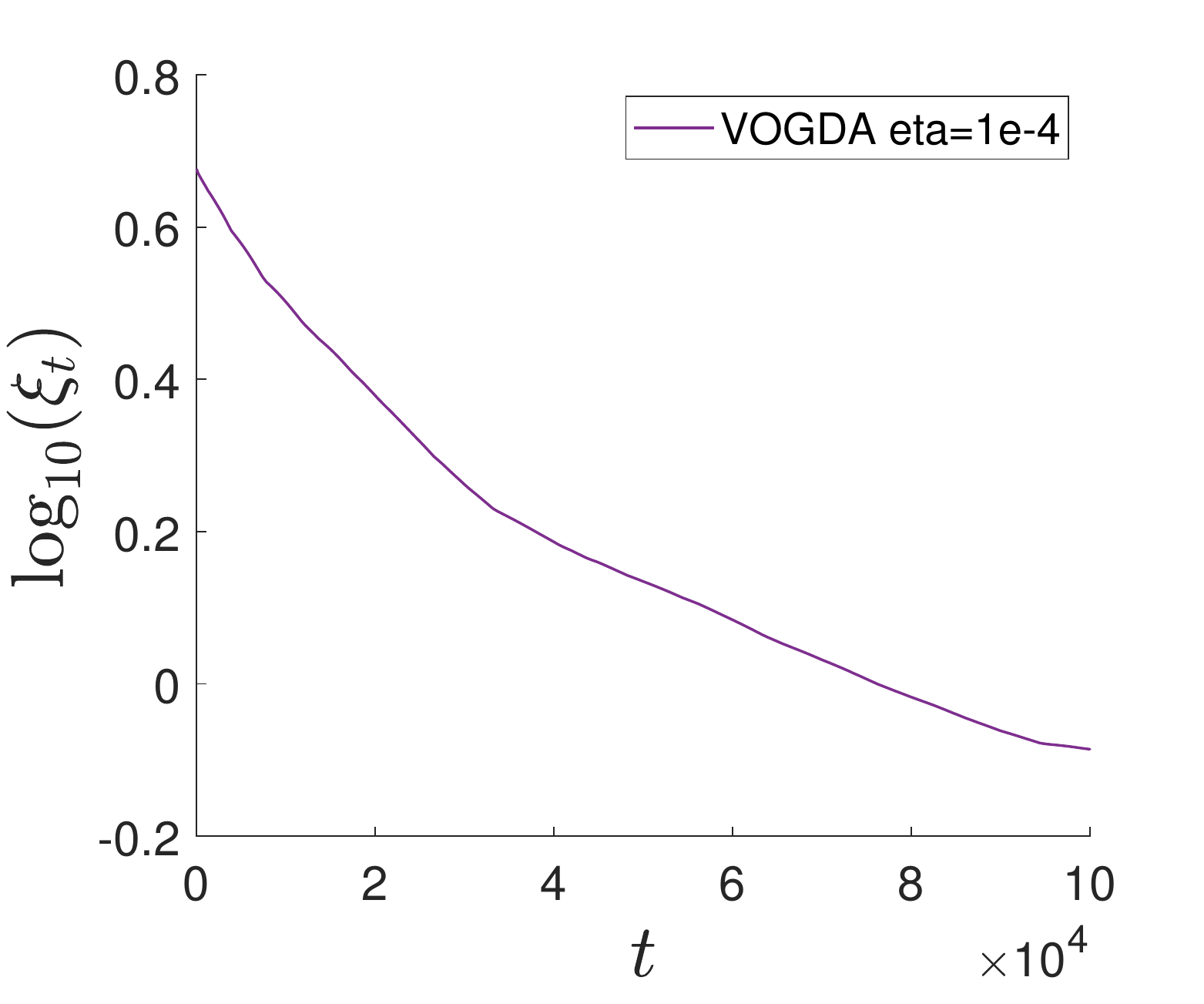}
    \caption{Experiments on Kuhn Leduc poker for \domwu (left) and \vogda (right) with small step sizes and more time steps.}
    \label{fig:fig2}
\end{figure}
}
\subsection{Description of the Games}
We briefly introduce the games in the experiments. 
Beside the rules of the games, we show the game size by providing $M,N$, $|\csimp^{\treeY}|$, and $|\csimp^{\treeY}|$ (recall $\G=\R^{M\times N}$).
\paragraph{Kuhn poker}
Introduced in \citep{kuhn1950simplified}, the deck for Kuhn poker contains three playing cards: King, Queen, and Jack. 
Each player is dealt one card, while the third card is unseen.
Then a betting process proceeds. 
Player $\x$ can check or raise, and then player $\y$ can also check or raise. 
Player $\x$ has a final round to call or fold if player $\x$ checks but player $\y$ raises in the previous round.
The player with the higher card wins the pot.
In this game, $M=N=13$, $|\csimp^{\treeX}|=|\csimp^{\treeY}|=6$.

\paragraph{Pursuit-evasion}
This is a search game considered in \citep{kroer2018robust}. 
Given a direct graph, player $\x$ controls an attacker to move in the graph, while player $\y$ controls two patrols, who are only allowed to move in their patrol areas.
The game has 4 rounds. 
In each round, the attacker and the patrols act simultaneously.
The attacker can move to any adjacent node or choose to wait in each round, with the final goal of going across the patrol areas and reach the goal nodes without being caught by the patrols. 
On the other hand, player $\y$'s goal is to let one of the patrols to reach the same node as the attacker in any round.
A patrol who visits a node that was previously visited by the attacker will know that the attacker was there if the attacker did not wait at that node (in order to clean up their traces).
In this game, $M=52$, $N=2029$, $|\csimp^{\treeX}|=34$, $|\csimp^{\treeY}|=348$.
\paragraph{Leduc poker}
Introduced in \citep{southey2005bayes}, the game is similar to Kuhn poker.
There are total 6 cards in the deck with two Kings, two Queens, and two Jacks.
Each player is dealt a private card while there is another unrevealed public card.
In the first round, player $\y$ bets after player $\x$ bets. After that the public card will be revealed and there is another betting stage.  
In a showdown stage, a player who has the same rank with the public card wins.
Otherwise, the player with the higher card wins.
In this game, $M=N=337$, $|\csimp^{\treeX}|=|\csimp^{\treeY}|=144$.
\subsection{Parameter Selection in the Dilated Regularizers}
We note that for the dilated regularizers, we use the unweighted version in the experiments. 
This is sufficient as \pref{lem:all one alpha} shows there always exists an assignment $\balpha$ such that $\psi^{\text{dil}}_{\balpha}$ is 1-strongly convex and $\balpha =\beta\cdot\mathbf{1} $ for some $\beta>0$, where $\mathbf{1}$ is the all-one vector over ${\csimp^\trplx}$.
In this way, the value of $\eta$ we refer to in \pref{fig:fig1} is actually $\eta/\beta$.
However, as we mentioned, the final $\eta$ used in the experiments may still be larger than what the theorems suggest. 
Showing similar results while allowing larger $\eta$ is an interesting future direction. 
\begin{lemma}\label{lem:all one alpha}
    For an assignment $\balpha$ such that $\psi^{\text{dil}}_{\balpha}$ is 1-strongly convex, $\psi^{\text{dil}}_{\balpha'}$ is also 1-strongly convex where $\balpha' = {\|\balpha\|_\infty}\mathbf{1}$.
\end{lemma}

\begin{proof}
    Recall the definition of regularizer $\psi^{\text{dil}}_{\balpha}(\z)$ in \pref{eq:dilated_reg}. Since each term $z_{{\seq(\infoi)}}\cdot\psi\left(\frac{\z_{\infoi}}{z_{\seq(\infoi)}}\right)$ is convex in $\z$ (thus with a non-negative Bregman divergence), $\breg{\psi^{\text{dil}}_{\balpha}}$ is an increasing function in any variable $\alpha_h$ ($h\in\csimp^\trplx$). 
    Therefore, we have
    \begin{align*}
        \frac{\|\z-\z'\|^2}{2}\le\breg{\psi^{\text{dil}}_{\balpha}}(\z,\z')\le \breg{\psi^{\text{dil}}_{\balpha'}}(\z,\z'),
    \end{align*}
    which completes the proof.
\end{proof}

\section{Omitted Details of \pref{sec:algs}}\label{app:sec4}
When introducing \vomwu in \pref{sec:algs}, we mention that $\vent$ is 1-strongly convex with respect to the 2-norm. 
In the following, we formally show this result.
Before that, we first show a technical lemma. 
\begin{lemma}
	For $u,v\in[0,1]$, the following inequalities hold:
	\begin{align}\label{eq: 1d entropy}
		\frac{(u-v)^2}{2}\le u\ln\left(\frac{u}{v}\right)-u+v\le \frac{(u-v)^2}{v}.
	\end{align}
\end{lemma}
\begin{proof}
	Define $f(u,v) = u\ln\left(\frac{u}{v}\right)-u+v-\frac{(u-v)^2}{2}$ and $g(u,v)=\frac{(u-v)^2}{v}-u\ln\left(\frac{u}{v}\right)+u-v$.
	To prove the claim, it is sufficient to show that the minimum of each of the two functions is zero. 
	Since both functions have the only critical point $(0,0)$, there is no extreme point in the interior.
	Also, it is straightforward to find a $(u,v)$ such that $f(u,v),g(u,v)>0$ in the interior. Thus, it remains to check if the boundary of the domain satisfies $f(u,v),g(u,v)\ge 0$.
	For $u=0$, we have
	\begin{align*}
		\frac{(0-v)^2}{2}=\frac{v^2}{2}\le v=\frac{(0-v)^2}{v}.
	\end{align*}
	The case for $v=0$  is trivial.
	For $u=1$, note that 
	\begin{align*}
		\frac{(v-1)^2}{2}\le v-1-\ln(v)\le\frac{(v-1)^2}{v}
	\end{align*}
	when $0\le v\le1$.
	For $v=1$, we have
	\begin{align*}
		\frac{(u-1)^2}{2}\le u\ln u-u+1\le(u-1)^2
	\end{align*}
	when  $0\le u\le1$.
	Therefore, we conclude $f(u,v),g(u,v)\ge 0$ and this finishes the proof.
\end{proof}
Now we are ready to give the result.
\begin{lemma}
	$\vent$ is 1-strongly convex with respect to the 2-norm.
\end{lemma}
\begin{proof}
	The result follows by the first inequality of \pref{eq: 1d entropy} and $0\le z_i\le1$ for all $\z\in\trplx$.
\end{proof}


\section{CFR-based Algorithms and Proof of \pref{thm: cfr diverges}}\label{app:cfr}
In this section, we first introduce CFR, CFR+, and their optimistic versions.
Then we show \pref{thm: cfr diverges} in \pref{app:thm 3 proof} and their empirical last-iterate divergence in \pref{app:rps}.
\subsection{CFR and Its Optimistic Version}
Given $P$-dimensional treeplex $\trplx$, loss vector $\bl\in \R^{P}$, and $\z\in\trplx$, we recursively define the value vector $\blp\in \R^{P}$, to be 
\[\lp_i=\ell_i+\sum_{g\in\csimp_i}\inner{\q_{g},\blp_{g}}\]
for every index $i$ (recall that $q_i=z_i/z_{p_i}$).
At time $t$, given $\q_t$, loss vector $\ell_t$, and its value vector $\blp_t$, denote $
    \reg_{t,j}^{g}=\inner{\q_{t,g},\blp_{t,g}}-L_{t,j}
$, $\Reg_{0,j}^{g}=0$, and
$
    \Reg_{t,j}^{g}=\Reg_{t-1,j}^{g}+\reg_{t,j}^{g}
$ for every simplex $g\in\csimp^\trplx$ and index $j\in\seqactn_g$.
In the literature, \emph{Counterfactual Regret Minimization} (CFR) \cite{zinkevich2007regret} refers to the algorithm running \emph{regret matching} on every simplex in a treeplex.
Specifically, on simplex $g$ at time $t+1$, regret matching plays arbitrarily when $\sum_{j\in\seqactn_g}\relu{\Reg_{t,j}^{g}}=0$, where $\relu{x}=\max(0,x)$; otherwise, it plays
\begin{align*}
	q_{t+1,i}=\frac{\relu{\Reg_{t,i}^{g}}}{\sum_{j\in\seqactn_g}\relu{\Reg_{t,j}^{g}}},
\end{align*}
for all $i$ in $\seqactn_g$.
In the two-player zero-sum setting, player $\x$ runs CFR on every simplex in treeplex $\treeX$ along with point $\x_t\in\treeX$ and loss vector $\blp_t^\x=\G\y_{t}$, and player $\y$ runs CFR on every simplex in $\treeY$ along with point $\y_t\in\treeY$ and loss vector $\blp_t^\y=-\G^\top\x_{t}$ for every time $t$.

The optimistic version of CFR \cite{farina2020faster} is running the optimistic version of {regret matching} on every simplex in a treeplex.
Specifically, the algorithm plays arbitrarily when $\sum_{j\in\seqactn_g}\relu{\Reg_{t,j}^{g}+\reg_{t,j}^{g}}=0$; otherwise,
\begin{align*}
	q_{t+1,i}=\frac{\relu{\Reg_{t,i}^{g}+\reg_{t,i}^{g}}}{\sum_{j\in\seqactn_g}\relu{\Reg_{t,j}^{g}+\reg_{t,j}^{g}}}.
\end{align*}

To get an approximate Nash equilibrium at time $t$, CFR and its optimistic version consider the average iterate, that is, they return
\[
(\overline{\x}_t,\overline{\y}_t)=\left(\frac{1}{t}\sum_{\tau=1}^t\x_\tau,\frac{1}{t}\sum_{\tau=1}^t\y_\tau\right).
\]

\subsection{CFR+ and Its Optimistic Version}
To introduce CFR+, we first introduce another regret-minimization algorithm on simplex, \emph{regret matching+}.
Similar to $\Reg_{t,j}^{g}$, we define $\Regp_{0,j}^{g}=0$ and 
\begin{align*}
    \Regp_{t,j}^{g}=\relu{\Regp_{t-1,j}^{g}+\reg_{t,j}^{g}},
\end{align*}
for every simplex $g\in\csimp^\trplx$ and index $j\in\seqactn_g$.
On simplex $g$ at time $t+1$, regret matching+ plays arbitrarily when $\sum_{j\in\seqactn_g}\relu{\Regp_{t,j}^{g}}=0$; otherwise, it plays
\begin{align*}
	q_{t+1,i}=\frac{\relu{\Regp_{t,i}^{g}}}{\sum_{j\in\seqactn_g}\relu{\Regp_{t,j}^{g}}}.
\end{align*}
CFR+ \cite{tammelin2014solving} refers to running \emph{regret matching+} on every simplex in a treeplex.
In the two-player zero-sum setting,
CFR+ usually refers to the one with \emph{alternating updates}.
Specifically, player $\x$ runs CFR+ on every simplex in $\csimp^\treeX$ with loss vector $\blp_t^\x=\G\y_{t}$ and player $\y$ runs CFR on every simplex in $\csimp^\treeY$ with loss vector $\blp_t^\y=-\G^\top\x_{t+1}$ (note that in the case with simultaneous updates, $\blp_t^\y=-\G^\top\x_{t}$).

The optimistic version of CFR+ \cite{farina2020faster} is running the optimistic version of {regret matching+} on every simplex in a treeplex.
Specifically, the algorithm plays arbitrarily when $\sum_{j\in\seqactn_g}\relu{\Regp_{t,j}^{g}+\reg_{t,j}^{g}}=0$; otherwise, it plays
\begin{align*}
	q_{t+1,i}=\frac{\relu{\Regp_{t,i}^{g}+\reg_{t,i}^{g}}}{\sum_{j\in\seqactn_g}\relu{\Regp_{t,j}^{g}+\reg_{t,j}^{g}}}.
\end{align*}
Regarding the averaging scheme, CFR+ usually refers to the version with \emph{linear averaging} to get an approximate Nash equilibrium at time $t$. 
Specifically, it returns
\[
(\overline{\x}_t,\overline{\y}_t)=\left(\frac{2}{t(t+1)}\sum_{\tau=1}^t\tau\cdot\x_\tau,\frac{2}{t(t+1)}\sum_{\tau=1}^t\tau\cdot\y_\tau\right).
\]
In summary, CFR+ and its optimistic version refer to running regret matching+ and optimistic regret matching+ on every simplex with alternating updates and linear averaging. 
%
\subsection{Proof of \pref{thm: cfr diverges}}\label{app:thm 3 proof}
\begin{proof}[Proof of \pref{thm: cfr diverges}]
    Note that in this instance, there is only one simplex $g^\x$ for player $\x$ and one simplex $g^\y$ for player $\y$.
    The game matrix $\G$ of the rock-paper-scissors is
    \begin{align*}
        \G = \begin{bmatrix}
            0 & -1 & 1\\
            1 & 0 & -1\\
            -1 & 1 & 0
        \end{bmatrix}=-\G^\top.
    \end{align*} 
    We consider the case when $\x_1=\y_1$.
    Recall that we have $\blp_t^\x=\G\y_{t}$ and $\blp_t^\y=-\G^\top\x_{t}$.
    Therefore, we know that 
    \begin{align*}
        \blp_1^\x=\G^\top\y_{1}=\G^\top\x_{1}=-\G^\top\x_{1}=\blp_1^\y,
    \end{align*}
    and
    \begin{align*}
        \reg_{1,j}^{g^\x}=\x_{1}^\top \blp_1^\x-\lp_{1,j}^{\x}=\x_{1}^\top \G\x_{1}-\lp_{1,j}^{\x}=-\lp_{1,j}^{\x}=-\lp_{1,j}^{\y}=\reg_{1,j}^{g^\y}
    \end{align*}
    for $j=1,2,3$.
    It is not hard to see that in this case, we have $\x_2=\y_2$, and thus $\x_t=\y_t$ for every $t$. Consequently, it is sufficient to focus on the updates of $\x_t$.
    For notional convenience, we write 
    \begin{align}
        \reg_t=\reg_t^{g^\x}=\G^{\top}\x_t= (x_{t,2}-x_{t,3},x_{t,3}-x_{t,1},x_{t,1}-x_{t,2})^\top,\label{eq: instant reg}\\
        \Reg_t=\Reg_t^{g^\x}=\Reg_{t-1}^{g^\x}+\reg_{t}=\sum_{\tau=1}^t\reg_{\tau},\label{eq: accumulated reg}
    \end{align}
    and thus
    \begin{align}
        x_{t+1,j}=\frac{\relu{\Reg_{t,j}}}{\relu{\Reg_{t,1}}+\relu{\reg_{t,2}}+\relu{\Reg_{t,3}}}\label{eq: rps xt update}
    \end{align}
    for $j=1,2,3$.
    We call distribution $\x_t$ \emph{imbalanced} if there exists a permutation $\lambda$ of $\{1,2,3\}$ such that
    \begin{align*}
        x_{t,\lambda(1)}\ge x_{t,\lambda(2)}\ge 0 = x_{t,\lambda(3)}.
    \end{align*} 
    We prove that if $\x_1$ is imbalanced, then every $\x_t$ is imbalanced. 
    Suppose at some time $t$, $\x_t$ is imbalanced and the corresponding $\lambda$ is the identity without loss of generality, that is,
    \begin{align}\label{eq:imbalanced}
        x_{t,1}\ge x_{t,2} \ge 0 = x_{t,3}.
    \end{align} 
    In this case, we know that $\Reg_{t-1,1}>0$. 
    By \pref{eq:imbalanced} and \pref{eq: instant reg}, we also know that $\reg_{t,1}\ge 0$, $\reg_{t,3}\ge 0$, $\reg_{t,2}< 0$, and $\Reg_{t,1}> 0$.
    Moreover, by \pref{eq: instant reg} and \pref{eq: accumulated reg}, we can get
    \begin{align*}
        \reg_{t,1}+\reg_{t,2}+\reg_{t,3}=0,~\Reg_{t,1}+\Reg_{t,2}+\Reg_{t,3}=0.
    \end{align*}
    Therefore, we have $\Reg_{t,2}+\Reg_{t,3}<0$, which means that at least one of $x_{t+1,2}$ and $x_{t+1,3}$ is zero.
    Moreover, we have
    \[
    \Reg_{t,1}=\Reg_{t-1,1}+\reg_{t,1}\ge\Reg_{t-1,1}\ge\Reg_{t-1,2}>\Reg_{t-1,2}+\reg_{t,2}=\Reg_{t,2},
    \]
    where the second inequality follows from $x_{t,1}\ge x_{t,2}$ and \pref{eq: rps xt update}.
    The inequalities above imply that $x_{t+1,1}> x_{t+1,2}$.
    Thus, we get one of the following three situations continues to hold at time $t+1$:
    \begin{align}
        x_{t+1,1}&\ge x_{t+1,2}\ge 0 = x_{t+1,3},\label{eq:imbalanced2.1}\\
        x_{t+1,1}&\ge x_{t+1,3}\ge 0 = x_{t+1,2},\label{eq:imbalanced2.2}\\
        x_{t+1,3}&\ge x_{t+1,1}\ge 0 = x_{t+1,2}.\label{eq:imbalanced2.3}
    \end{align}  
    If \pref{eq:imbalanced2.1} holds at time $t+1$, the same argument implies one of the three arguments above continues to hold;
    otherwise, if at some time step $\tau>t$, \pref{eq:imbalanced2.2} holds, that is,
    \begin{align}\label{eq:imbalanced3}
        x_{\tau,1}\ge x_{\tau,3}\ge 0 = x_{\tau,2}.
    \end{align}
    Similarly, we know that $\reg_{\tau,1}\le 0,\reg_{\tau,2}\le 0$ and $\reg_{\tau,3}\ge 0$, and $x_{\tau+1,2}=0$. 
    Thus, we get either \pref{eq:imbalanced3} continues to hold at time $\tau+1$ or
    \begin{align*}
        x_{\tau+1,3}\ge x_{\tau+1,1}\ge 0 = x_{\tau+1,2}
    \end{align*}
    holds, which is exactly the same permutation in \pref{eq:imbalanced2.3}.
    With similar arguments, we know that for every imbalanced distribution, either the same permutation holds in the next round, or it transits to another imbalanced distribution with another permutation.  
    Note that the average iterate of the sequence $\{\x_t\}_t$ converges to the uniform distribution as CFR is a no-regret algorithm, so $\x_t$ never converges to any imbalanced distribution.
    Therefore, we conclude that $\x_t$ diverges if the algorithm starts from $x_1=y_1$ being an arbitrary imbalanced distribution.
\end{proof}
\subsection{Experiments}\label{app:rps}
Besides \pref{thm: cfr diverges}, we empirically observe divergence of CFR, CFR+ with simultaneous updates, and their optimistic versions in the rock-paper-scissors game. 
The results are shown in \pref{fig:rps}.
We remark that in these experiments, we consider CFR+ with simultaneous updates instead of the more commonly used ones (alternating updates). 
In fact, we observe that with alternating updates, the optimistic CFR+ empirically has last-iterate convergence in some matrix games.
As all of our theoretical results are for simultaneous updates, the theoretical justification of this observation is beyond the scope of this paper, but it is an interesting direction for future works.
\begin{figure}
    \centering
    \includegraphics[width=0.39\textwidth]{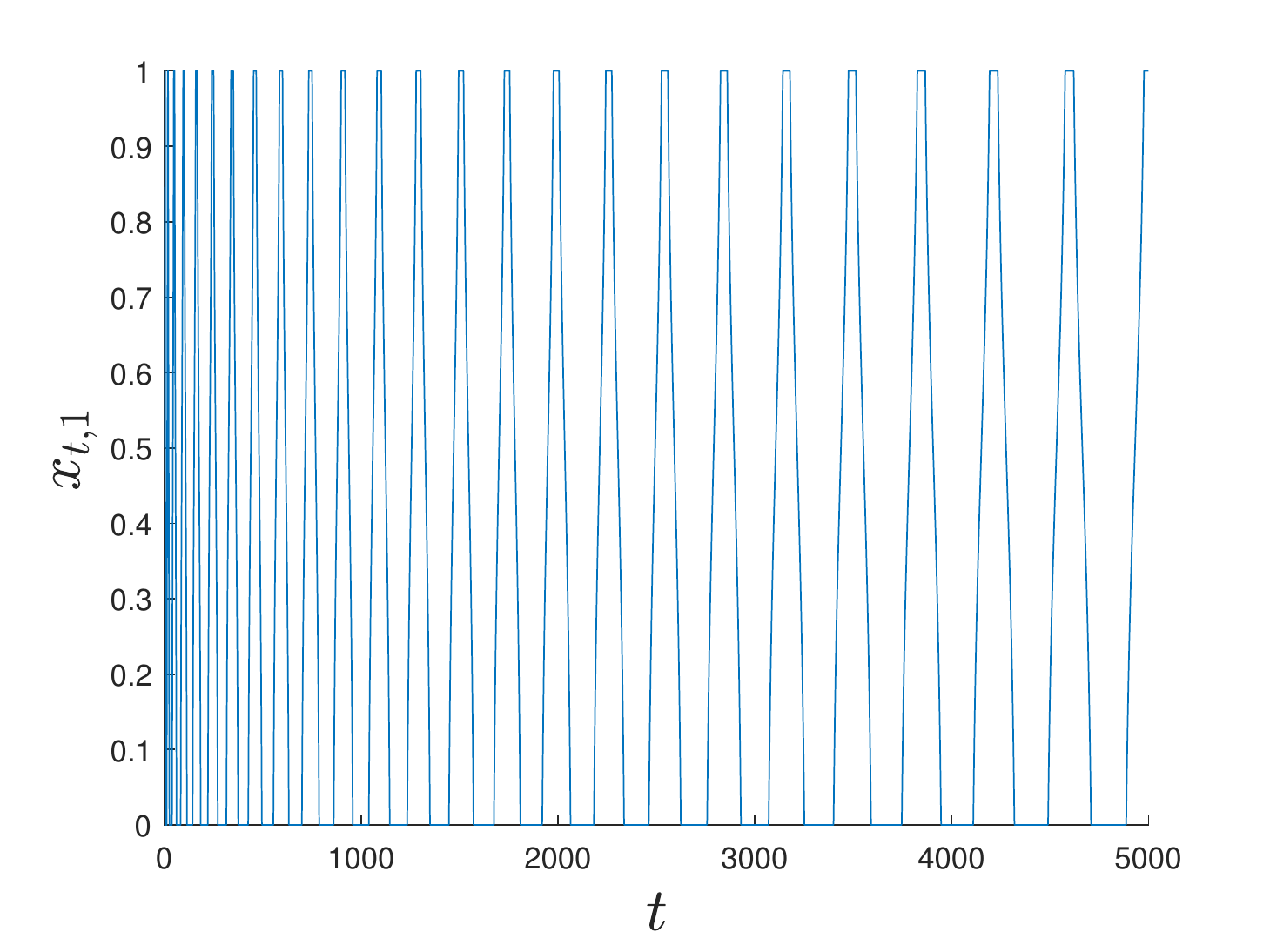}\quad
    \includegraphics[width=0.39\textwidth]{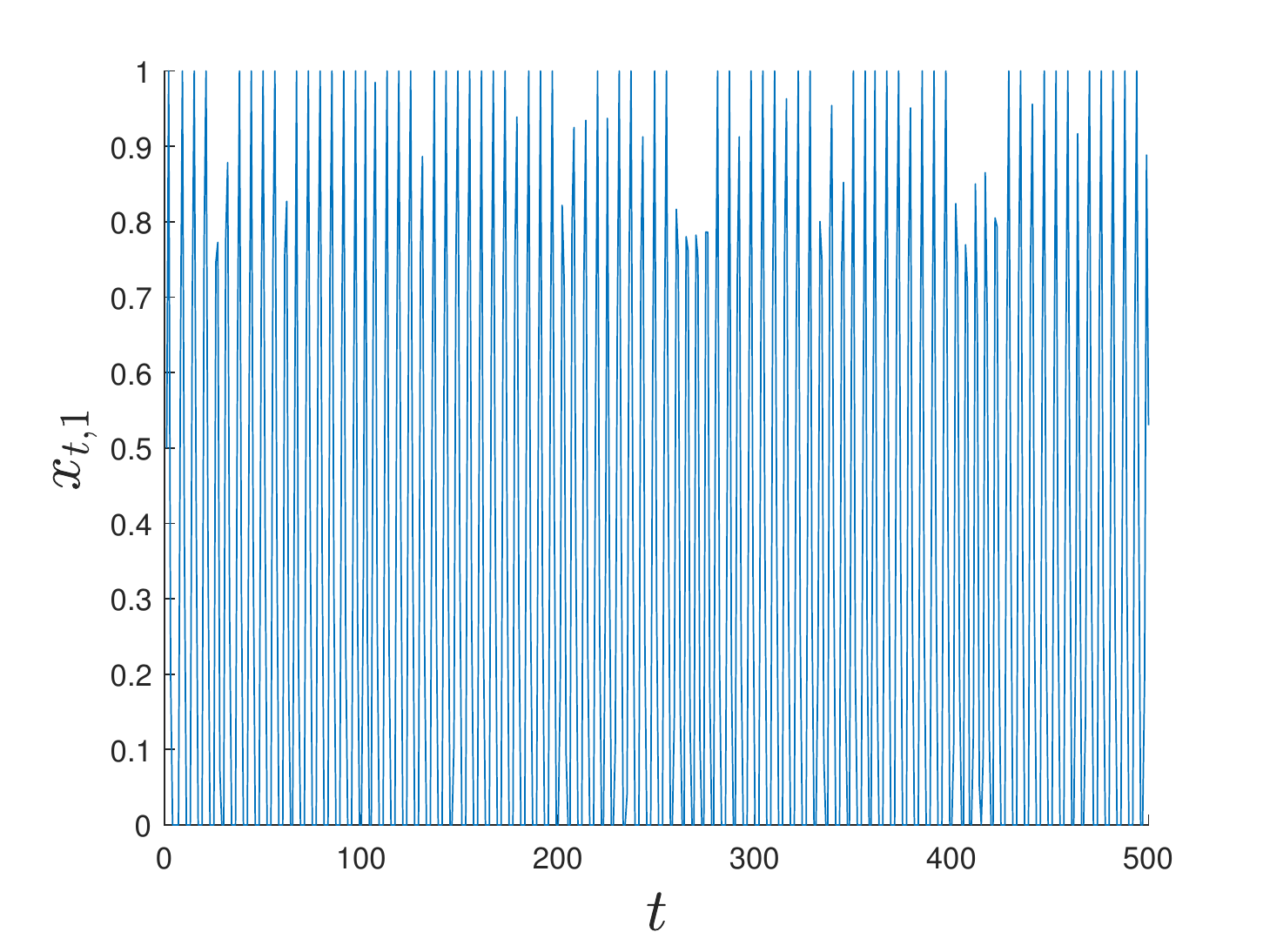}\\
    \includegraphics[width=0.39\textwidth]{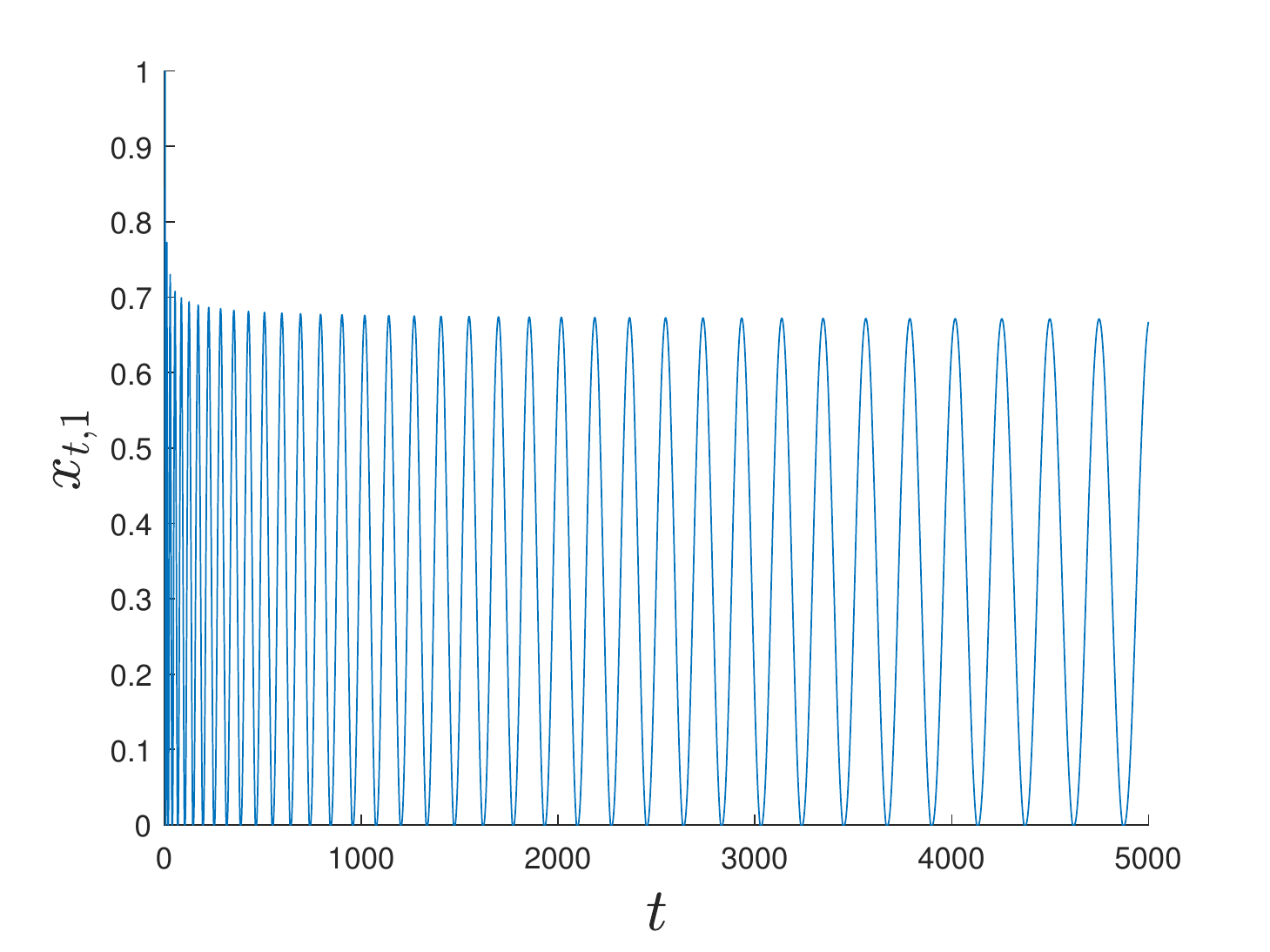}\quad
    \includegraphics[width=0.39\textwidth]{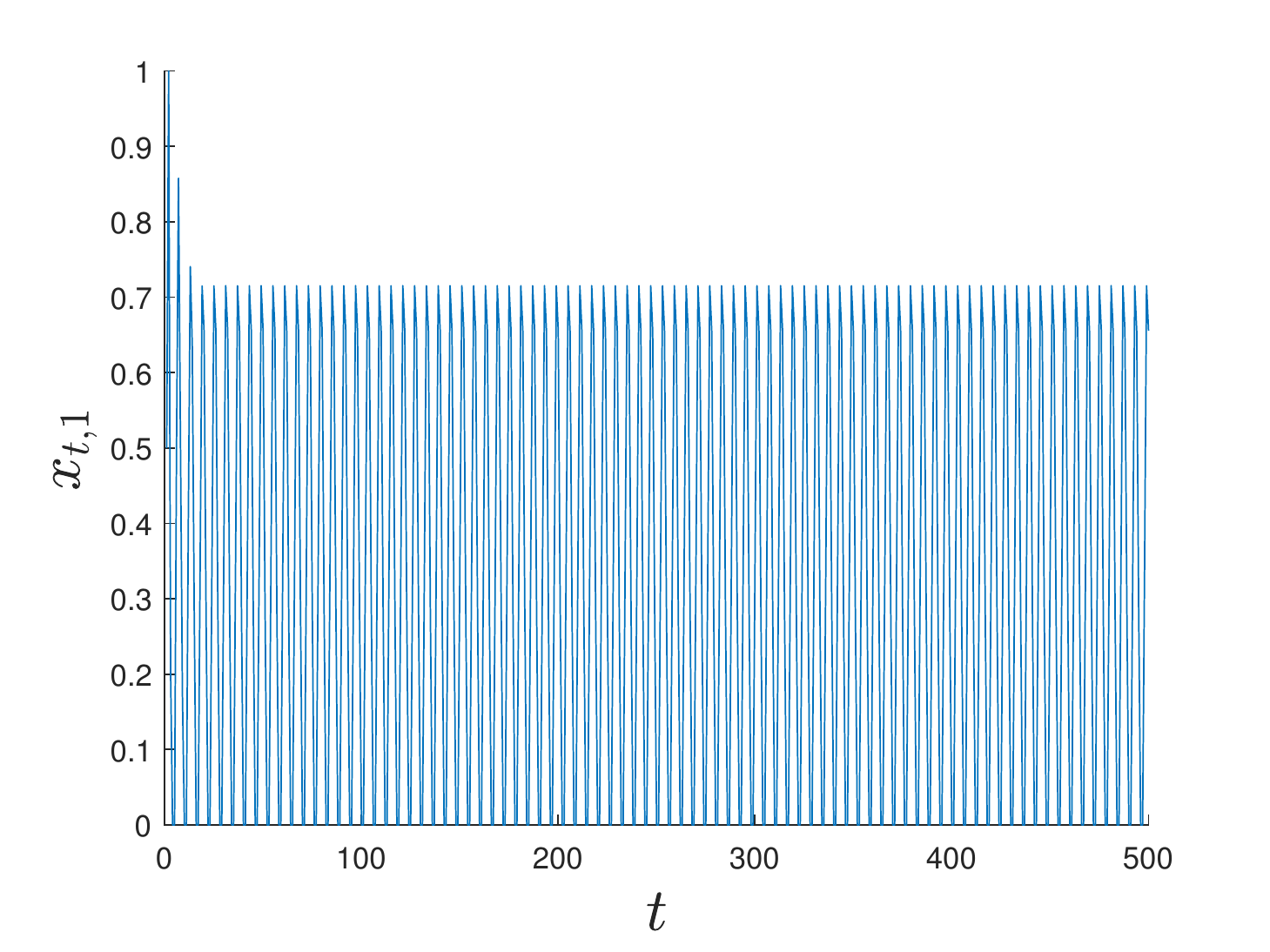}
    \caption{Last-iterate divergence of CFR, CFR+ with simultaneous updates, and their optimistic versions in the rock-paper-scissors game. The first row is the CFR algorithms while the second row is the CFR+ algorithms. For both rows, we put the vanilla version on the left and the optimistic version on the right. All algorithms start with $\x_1=\y_1=(1,0,0)^\top$.}
    \label{fig:rps}
\end{figure}
\section{Proofs of \pref{thm: vomwu rate} and \pref{thm: domwu rate}}\label{app:OMWU_proofs}

In this section, we show the proof of \pref{thm: vomwu rate} in \pref{app:thm6} and the proof of \pref{thm: domwu rate} in \pref{app:thm7}.
We generally follow the outline in \pref{sec:analysis of thm 6 7}.
We discuss the technical difficulty to get a convergence rate for \dogda in \pref{app:dogda}.
Throughout the section, we assume that $\G$ has a unique Nash equilibrium $\z^*=(\x^*,\y^*)$ (call it the uniqueness assumption).
For the sake of analysis, we equivalently define $\dent$ as \[\dent(\z) = \sum_{i}\alpha_{\infoi(i)}z_i \ln \frac{z_i}{\sum_{j\in\Omega_{\infoi(i)}}z_j}.\]
Note that under this definition, we have 
\begin{align}\label{eq:partial derivative}
    \frac{\partial\dent(\z)}{\partial z_i}=\alpha_{\infoi(i)}\left[\ln \left(\frac{z_i}{\sum_{j\in\Omega_{\infoi(i)}}z_j}\right)+1-\sum_{k\in\Omega_{\infoi(i)}}\frac{z_k}{\sum_{j\in\Omega_{\infoi(i)}}z_j}\right]=\alpha_{\infoi(i)}\ln q_i,
\end{align}
and
\begin{align}\label{eq:dent breg}
    \breg{\dent}(\z,\z')=\sum_{i}\alpha_{\infoi(i)}\left[z_i \ln (q_i)-z_i'\ln(q_i')-(z_i-z_i')\ln(q_i')\right]=\sum_{i}\alpha_{\infoi(i)}z_i\ln\frac{q_i}{q_i'}.
\end{align}
\subsection{Strict Complementary Slackness and Proof of \pref{eq: zeta ge l2 norm} in EFGs}\label{app:slackness}
In this subsection, we prove \pref{eq: zeta ge l2 norm}, which is restated in \pref{lem:eq6}.
\begin{lemma}\label{lem:eq6}
    Under the uniqueness assumption, for both \vomwu and \domwu, there exists some constant $C_{12}>0$ that depends on the game, $\eta$, and $\zp_1$ such that for any $t$,
    \begin{align}
        \zeta_t = \KL(\zp_{t+1}, \z_{t}) +\KL(\z_t, \zp_t) \ge C_{12}\|\z^*-\zp_{t+1}\|^2.\label{eq:eq6}
    \end{align}    
\end{lemma}

Before proving this lemma, we show some useful properties of EFGs. The first one is the strict complementary slackness.
\subsubsection{Strict Complementary Slackness}\label{app:slack}
We formulate the minimax problem in \pref{eq: minimax form} as a linear program.
This is a standard procedure in the literature (see, for example, \citep[Section 3.11]{nisan2007algorithmic}), but we show its derivation here for completeness.
Based on \pref{def:treeplex}, we have $\x\in\treeX$ if $\x$ satisfies $x_i\ge 0$ for every $i$ and
\begin{align}
	x_{\emp}=1,\qquad\sum_{i\in\seqactn_\infoi}x_{i}=x_{\seq(\infoi)},~ \forall \infoi\in\infoX. \label{eq: index form constraints}
\end{align}
We can write the constraints in \pref{eq: index form constraints} using a matrix $\matX\in\R^{\left(\left|\infoX\right|+1\right)\times M}$ such that $\x\in\treeX$ if and only if $\matX \x=\ve,~\x\ge \zero$, where all entries in $\ve$ are zero except for $e_{\emp}=1$, which corresponds to the constraint $x_{\emp}=e_{\emp}=1$.
Similarly, for player $\y$, we have the constraint matrix $\matY$ and vector $\vf$.
Consequently, we write the best response of $\x$ to a fixed $\y$ to as the following linear program:
\begin{align*}
	\min_{\x} \x^\top(\G\y),~\text{subject to}~\matX\x=\ve,~\x\ge \zero.
\end{align*}
The dual of this linear program is
\begin{align*}
	\max_{\bV} \ve^\top\bV,~\text{subject to}~\matX^\top \bV\le\G\y,
\end{align*}
where $\bV\in\R^{|\infoX|+1}$.
Recall that all entries in $\ve$ are zero except for $e_{\emp}=1$ so the objective $\ve^\top\bV=V_0$.
On the other hand, player $\y$ tries to maximize $\x^\top\G\y$, that is, maximize $V_0$ by the strong duality.
Therefore, every $\y^*\in\treeY^*$ is a solution to the following maximin problem  
\begin{align}\label{eq: maximin opt}
	\max_{\bV,\y} V_0,~\text{subject to}~\matX^\top \bV\le\G\y,~\matY\y=\vf,~\y\ge\zero,
\end{align}
which is also a linear program.
The dual of this linear problem is 
\begin{align}\label{eq: minimax opt}
	\min_{\bU,\x} U_0,~\text{subject to}~\matY^\top \bU\ge\G^\top\x,~\matX\x=\ve,~\x\ge\zero.
\end{align}
It is also not hard to see that every $\x^*\in\treeX$ is a solution to this dual.
We conclude that \pref{eq: maximin opt} and \pref{eq: minimax opt} are primal-dual linear programs of the minmax problem in \pref{eq: minimax form}. 
Given an optimal solution pair $\x^*,\y^*$ along with $\bV^*,\bU^*$, by the complementary slackness, we have some slackness variables $\w^*\in\R^{ M},\s^*\in\R^{ N}$ such that $\x^*\odot\w^*=\zero$, $\y^*\odot\s^*=\zero$ ($\odot$ denotes the element-wise product), and
\begin{align}\label{eq: complementary slackness}
\matX^\top \bV^*-\G\y^*+\w^*=\zero,\quad\matY^\top \bU^*-\G^\top\x^*-\s^*=\zero,\quad\w^*,\s^*\geq\zero.
\end{align}
Thus, \pref{eq: complementary slackness} implies that for every index $i$ of player $\x$, 
\begin{align}\label{eq: value recursion optimal}
	V^*_{\infoi(i)}-\sum_{g\in\csimp_i}V^*_g-(\G\y^*)_i=(\matX^\top \bV^*)_i-(\G\y^*)_i=(\matX^\top \bV^*-\G\y^*)_i=-w^*_i.
\end{align}
Additionally, the \emph{strict complementary slackness} (see, for example,~\citep[Theorem 10.7]{vanderbei2015linear}) ensures that there exists an optimal solution $(\x^*,\y^*)$ such that
\begin{align}\label{eq: strict complementary slackness}
	\x^*+\w^*>\zero,\quad\y^*+\s^*>\zero.
\end{align}
Under the uniqueness assumption, the strict complementary slackness must hold for the unique optimal solution $(\x^*,\y^*)$. 
Therefore, when $x^*_i>0$, we have $w^*_i=0$, which means that
\pref{eq: value recursion optimal} implies
\begin{align}\label{eq: value recursion 2}
	V^*_\infoi=(\G\y^*)_i+\sum_{g\in\csimp_i}V^*_g,\quad \forall h\in\infoX,~i\in\seqactn_h,~ x^*_i>0;
\end{align}
otherwise, if $x^*_i=0$ and $w^*_i>0$, by \pref{eq: value recursion optimal}, we have 
\begin{align}\label{eq: value recursion suboptimal}
	V^*_\infoi<(\G\y^*)_i+\sum_{g\in\csimp_i}V^*_g,\quad \forall h\in\infoX,~i\in\seqactn_h,~x^*_i=0.
\end{align}
Note that for every terminal index $i$, ${\csimp_i}$ is empty and we set the term $\sum_{g\in\csimp_i}V^*_g$ zero correspondingly.
The case is similar for $\bU^*$ and $\G^\top\x^*$.
We summarize this result as the following lemma.

\begin{lemma}\label{lem: gap lemma}
    Under the uniqueness assumption, we have
     \begin{align*}
		\sum_{g\in\csimp_i}V^*_g+(\G\y^*)_i &= V^*_{\infoi(i)}  &&\forall i\in \supp(\x^*), \\
		\sum_{g\in\csimp_i}V^*_g+(\G\y^*)_i &> V^*_{\infoi(i)}  &&\forall i\notin \supp(\x^*), \\
		\sum_{g\in\csimp_j}U^*_g+(\G^\top\x^*)_j &= U^*_{\infoi(j)}  &&\forall j\in\supp(\y^*), \\
		\sum_{g\in\csimp_j}U^*_g+(\G^\top\x^*)_j &< U^*_{\infoi(j)} &&\forall j\notin\supp(\y^*).
     \end{align*}
\end{lemma}
\subsubsection{Some Problem-dependent Constants}
After introducing the strict complementary slackness, we are ready to introduce some problem-dependent constants.
Note that by \pref{lem: gap lemma}, we have the following constant $\xi>0$. 
\begin{definition}\label{def: xi}
	Under the uniqueness assumption, we define
	\begin{align*}
    \xi\triangleq \min\left\{ \min_{i\notin\supp(\x^*)} \sum_{g\in\csimp_i}V^*_g+(\G\y^*)_i - V^*_{\infoi(i)}, ~\min_{j\notin\supp(\y^*)} U^*_{\infoi(j)} -(\G^\top\x^*)_j-\sum_{g\in\csimp_j}U^*_g \right\} \in (0,2]. 
    \end{align*}
\end{definition}
Note that $\xi\le2$ follows from the fact that for any information set $h\in\csimp^\treeX$, indices $i,j\in\seqactn_h$ such that $i\notin\supp(\x^*)$ and $j\in\supp(\x^*)$, by \pref{lem: gap lemma}, we have
\[
\xi\le  \sum_{g\in\csimp_i}V^*_g+(\G\y^*)_i - V^*_{h}=(\G\y^*)_i-(\G\y^*)_j\le 2.
\]

Below, we define
$\mathcal{V}^{*}(\mathcal{Z})=\mathcal{V}^{*}(\mathcal{X})\times \mathcal{V}^{*}(\mathcal{Y})$, where 
\[
\mathcal{V}^{*}(\mathcal{X})\triangleq\{\x: \x\in\calX,~\supp(\x)\subseteq \supp(\x^*)\}
\] 
and 
\[
\mathcal{V}^{*}(\mathcal{Y})\triangleq\{\y: \y\in\calY,~ \supp(\y)\subseteq \supp(\y^*)\}.
\]
\begin{definition}\label{def: cx cy}
    \begin{align*}
        c_x \triangleq \min_{{\x\in\calX \backslash \{\x^*\}}}  \max_{\y\in \mathcal{V}^{*}(\mathcal{Y})} \frac{(\x-\x^*)^\top \G \y}{\|\x-\x^*\|_1}, \qquad c_y \triangleq \min_{{\y\in\calY\backslash\{\y^*\}}}  \max_{\x\in \mathcal{V}^{*}(\mathcal{X})}  \frac{\x^\top \G (\y^*- \y)}{\|\y^*- \y\|_1}.
    \end{align*} 
\end{definition}

The following lemma shows that $c_x$ and $c_y$ are well-defined even though the outer minimization is over an open set. 
The proof generally follows \citep[Lemma 14]{wei2021linear} but requires the results derived in \pref{app:slack}.

\begin{lemma}\label{lem: cx-well-defined}
    $c_x$ and $c_y$ are well-defined, and $0<c_x,c_y\le \diam$.
\end{lemma}
\begin{proof}[Proof]
    We first show $c_x$ and $c_y$ are well-defined. To simplify the notations, we define $x_{\min}^*\triangleq\min_{i\in \supp(\x^*)}x_i^*$ and $\calX'\triangleq\{\x: \x\in \calX,~\|\x-\x^*\|_1\ge x_{\min}^*\}$, and define $y_{\min}^*$ and $\calY'$ similarly. We will show that
\begin{align*}
    c_x = \min_{\x\in \calX'}  \max_{\y\in\mathcal{V}^{*}(\mathcal{Y})} \frac{(\x-\x^*)^\top \G \y}{\|\x-\x^*\|_1},\quad c_y = \min_{\y\in\calY'}  \max_{\x\in\mathcal{V}^{*}(\mathcal{X})}  \frac{\x^{\top} \G (\y^*- \y)}{\|\y^*- \y\|_1}, 
\end{align*}
which are well-defined as the outer minimization is now over a closed set.
To prove the equality for $c_x$, it suffices to show that for any $\x\in \calX$ such that $\x\ne \x^*$ and $\|\x-\x^*\|_1 < x_{\min}^*$, there exists $\x'\in \calX$ such that $\|\x'-\x^*\|_1 = x_{\min}^*$ and
\begin{align}\label{eq: mapping}
     \frac{(\x-\x^*)^\top \G \y}{\|\x-\x^*\|_1} = \frac{(\x'-\x^*)^\top \G \y}{\|\x'-\x^*\|_1}, \;\forall \y.
\end{align}

In fact, we can simply choose $\x' = \x^*+(\x-\x^*)\cdot \frac{x_{\min}^*}{\|\x-\x^*\|_1}$. We first argue that $\x'$ is still in $\calX$. 
For each index $j$, if $x_j-x^*_j\ge 0$, we surely have $x'_j \ge x_j^* +0\ge 0$; otherwise, $x_j^* > x_j\geq 0$ and thus $j\in \supp(\x^*)$ and $x^*_j\ge x_{\min}^*$, which implies $x'_j \geq x^*_j - |x_j-x^*_j|\cdot \frac{x_{\min}^*}{\|\x-\x^*\|_1} \ge x^*_j - x_{\min}^* \ge 0$. 
In addition, for any $\infoi\in\infoX$, 
\begin{align*}
    \sum_{j\in \seqactn_{\infoi}} x'_{j} &=  \frac{x_{\min}^*}{\|\x-\x^*\|_1}\cdot\sum_{j\in \seqactn_{\infoi}} x_{j}+\left(1- \frac{x_{\min}^*}{\|\x-\x^*\|_1}\right)\sum_{j\in \seqactn_{\infoi}} x^*_{j}\\
    & =\frac{x_{\min}^*}{\|\x-\x^*\|_1}\cdot x_{\seq(\infoi)}+\left(1- \frac{x_{\min}^*}{\|\x-\x^*\|_1}\right) x_{\seq(\infoi)}^*\\
    &=x'_{\seq(\infoi)}.
\end{align*}
Therefore, we conclude $\x'\in \calX$. 
Moreover, according to the definition of $\x'$, $\|\x'-\x^*\|_1 = x_{\min}^*$ holds. Also, since $\x^*-\x$ and $\x^*-\x'$ are parallel vectors, \pref{eq: mapping} is satisfied. The  arguments above show that the $c_x$ in \pref{def: cx cy} is a well-defined real number. The case of $c_y$ is similar.

Now we show $0<c_x,c_y\le \diam$. The fact that $c_x, c_y\leq \diam$ is a direct consequence of the definitions. Below, we 
use contradiction to prove that $c_y>0$. First, if $c_y < 0$, then there exists $\y\neq \y^*$ such that $ \x^{*\top}\G\y^*< \x^{*\top}\G \y$. This contradicts with the fact that $(\x^*, \y^*)$ is the equilibrium. 
    
On the other hand, if $c_y=0$, then there is some $\y\neq \y^*$ such that 
\begin{equation}\label{eq:some_condition}
\max_{\x\in \mathcal{V}^{*}(\mathcal{X})}\x^{\top}\G(\y^*-\y) = 0.
\end{equation}
 Consider the point $\y' = \y^* + \frac{\xi}{2N}(\y-\y^*)$ (recall the definition of $\xi$ in \pref{def: xi} and that $0< \xi \leq 2$),
    which is a convex combination of $\y^*$ and $\y$, and hence $\y'\in\treeY$.
    Then, for any ${\x\in\calX}$,
    \begin{align*}
        \x^{\top}\G\y'&= \sum_{i\notin\supp(\x^*)} x_i(\G\y')_i +  \sum_{i\in\supp(\x^*)} x_i(\G\y')_i\\ 
        &\geq \sum_{i\notin\supp(\x^*)} \big(x_i(\G\y^*)_i- x_i\|\y'-\y^*\|_1\big)  +\sum_{i\in\supp(\x^*)}\left(\frac{\xi}{2}\cdot x_i(\G(\y-\y^*))_i + x_i(\G\y^*)_i\right) \tag{using $G_{ij}\in [-1, 1]$ for the first part and $\y'=\y^*+\frac{\xi}{2N}(\y-\y^*)$ for the second} \\
        &\geq \sum_{i\notin\supp(\x^*)} \big(x_i(\G\y^*)_i- x_i\|\y'-\y^*\|_1\big)  +\sum_{i\in\supp(\x^*)} x_i\left(V^*_{\infoi(i)}-\sum_{g\in\csimp_i}V^*_g\right),
\end{align*}
where the last inequality is due to \pref{eq:some_condition} and \pref{lem: gap lemma}.
We continue to bound the terms above, which are bounded by
\begin{align*}
       &\geq \sum_{i\notin\supp(\x^*)} \big(x_i\left((\G\y^*)_i- \xi\right) \big)  +\sum_{i\in\supp(\x^*)} x_i\left(V^*_{\infoi(i)}-\sum_{g\in\csimp_i}V^*_g\right)   \tag{using $\y'-\y^*=\frac{\xi}{2N}(\y-\y^*)$ and $\|\y-\y^*\|_1 \leq 2N$} \\
        &\geq \sum_{i\notin\supp(\x^*)} x_i\left(V^*_{\infoi(i)}-\sum_{g\in\csimp_i}V^*_g\right) +  \sum_{i\in\supp(\x^*)} x_i\left(V^*_{\infoi(i)}-\sum_{g\in\csimp_i}V^*_g\right)  \tag{by the definition of $\xi$} \\
        &=\sum_{i} x_i\left(V^*_{\infoi(i)}-\sum_{g\in\csimp_i}V^*_g\right).
    \end{align*}
    The last term can be writen as the matrix form $\x^\top\matX^\top\bV^*=\be^\top\bV^*= V_\emp^*$.
    This shows that $\min_{\x\in\calX} \x^{\top}\G\y'\ge V_\emp^*$, that is, $\y' \neq \y^*$ is also a maximin point, contradicting the uniqueness assumption. Therefore, $c_y>0$ has to hold, and so does $c_x>0$ by the same argument.
\end{proof} 
We continue to define more constants in the following.
\begin{definition}\label{def: epsilon definition}
Define constants $\zmin\triangleq\min_i{\zpp_{1,i}}\in (0,1]$,
\begin{align*}
     &\epsv\triangleq\min_{j\in\supp({\z^*})}  \exp\left(-\frac{ P(1+\ln(1/\zmin))}{z_j^*}\right)\in(0,1),\\
     &\epsd\triangleq\min_{j\in\supp({\z^*})}  \left\{{z^*_j}\cdot \exp\left(-\frac{\|\alpha\|_\infty P^2\ln\left(1/\zmin\right)}{z_j^*}\right)\cdot\left(\frac{3}{4}\right)^P\right\}\in(0,1).
\end{align*}
\end{definition}
For all $i\in\supp(\z^*)$, we will show that $\epsv$ is a lower bound of $\zpp_{t,i}$ for \vomwu, while $\epsd$ is a lower bound of $\zpp_{t,i}$ and $z_{t,i}$ for \domwu.
We show the results in \pref{lem: smallest index} and defer the proof there. 
\subsubsection{Proof of \pref{lem:eq6}} We are almost ready to prove \pref{lem:eq6}. 
Before that, we first show the following auxiliary lemma. 
\begin{lemma}
    \label{lem:them5prime}
    For any $\z\in \calZ$, we have
    \[
    \max_{\z'\in \mathcal{V}^{*}(\mathcal{Z})}F(\z)^\top(\z-\z')\ge C\|\z^*-\z\|_1,
    \] 
    for $C = \min\{c_x, c_y\} \in (0, 1]$.  
\end{lemma}

\begin{proof}
    Recall that $V^*_0={\x^*}^\top \G\y^*$ is the game value and note that
    \begin{align*}
    \max_{\z'\in \mathcal{V}^{*}(\mathcal{Z})}F(\z)^\top(\z-\z')&= \max_{\z'\in\mathcal{V}^*(\calZ)}(\x-\x')^\top \G \y + \x^\top \G (\y' - \y) =   \max_{\z'\in\mathcal{V}^*(\calZ)} -\x^{\prime\top} \G \y + \x^\top \G \y' \\
    &=\max_{\x'\in \mathcal{V}^{*}(\mathcal{X})}\left( V_\emp^* -\x'^\top \G\y\right)+\max_{\y'\in \mathcal{V}^{*}(\mathcal{Y})} \left(\x^\top \G\y'- V_\emp^*\right)\\
    &=\max_{\x'\in \mathcal{V}^{*}(\mathcal{X})}\x'^\top \G(\y^*-\y)+\max_{\y'\in \mathcal{V}^{*}(\mathcal{Y})}(\x-\x^*)^\top \G\y'\\
    &\ge c_y\|\y^*-\y\|_1+c_x\|\x^*-\x\|_1\tag{by \pref{def: cx cy}}\\
    &\ge \min\{c_x,c_y\}\|\z^*-\z\|_1,
    \end{align*}
    where the third equality is due to \pref{eq: complementary slackness}, $\x'\in \mathcal{V}^{*}(\mathcal{X})$, and
    \[
    V^*_0=\ve^\top \bV^*=(\x'^\top\matX^\top) \bV^*=\x'^\top(\matX^\top \bV^*)=\x'^\top(\G\y^*);
    \]
    the case for $\y'$ is similar.
    This completes the proof.
\end{proof}
Now we show the proof of \pref{lem:eq6}.

\begin{proof}[Proof of \pref{lem:eq6}]
    Below we consider any $\z' \in\calZ$ such that $\supp(\z')\subseteq\supp(\z^*)$, that is, $\z'\in\mathcal{V}^{*}(\mathcal{Z})$. Considering \pref{eq: oomd update} , and using the first-order optimality condition of $\zp_{t+1}$, we have 
\begin{align}\label{eq: 1st opt}
(\nabla\psi(\zp_{t+1}) - \nabla\psi(\zp_t) + \eta F(\z_t))^\top (\z'-\zp_{t+1}) \geq 0.
\end{align}
Rearranging the terms and we get 
\begin{align}
 \eta F\left(\z_{t}\right)^{\top}\left(\zp_{t+1}-\z'\right) \le \left(\nabla\psi(\zp_{t+1})-\nabla\psi(\zp_{t})\right)^{\top}\left(\z'-\zp_{t+1}\right).   \label{eq: tmp eq}
\end{align}
The left hand side of \pref{eq: tmp eq} is lower bounded as
\begin{align*}
 \eta F\left(\z_{t}\right)^{\top}\left(\zp_{t+1}-\z'\right)
 &=\eta F\left(\zp_{t+1}\right)^{\top}\left(\zp_{t+1}-\z'\right)+\eta\left(F\left(\z_{t}\right)-F\left(\zp_{t+1}\right)\right)^{\top}\left(\zp_{t+1}-\z'\right) \\ 
 &\geq \eta F\left(\zp_{t+1}\right)^{\top}\left(\zp_{t+1}-\z'\right) - \eta\|F\left(\z_{t}\right)-F\left(\zp_{t+1}\right)\|_\infty\|\zp_{t+1}-\z'\|_1 \\ 
 & \geq \eta F\left(\zp_{t+1}\right)^{\top}\left(\zp_{t+1}-\z'\right)-2P\eta \left\|\z_{t}-\zp_{t+1}\right\|_1  \\ 
 & \geq \eta F\left(\zp_{t+1}\right)^{\top}\left(\zp_{t+1}-\z'\right)-\frac{1}{4}\left\|\z_{t}-\zp_{t+1}\right\|_1; \tag{$\eta \leq 1/(8P)$}
\end{align*}
When $\psi=\vent$, we have
\begin{align}\label{eq: vomwu gradient}
    (\nabla\vent(\zp_{t+1})-\nabla\vent(\zp_{t}))_i=(1+\ln\zpp_{t+1,i})-(1+\ln\zpp_{t,i})=\ln \frac{\zpp_{t+1,i}}{\zpp_{t,i}}.
\end{align}
On the other hand, when $\psi=\dent$, by \pref{eq:partial derivative}, we have 
\begin{align}\label{eq: domwu gradient}
    \left(\nabla\dent(\zp_{t+1})-\nabla\dent(\zp_{t})\right)_i=\alpha_i\ln \frac{\qpp_{t+1,i}}{\qpp_{t,i}}.
\end{align}
Therefore, the right hand side of \pref{eq: tmp eq} for \vomwu is upper bounded by 
\begin{align*}
&\left(\nabla\vent(\zp_{t+1})-\nabla\vent(\zp_{t})\right)^{\top}\left(\z'-\zp_{t+1}\right)\\
&=\sum_i \left(z'_i-\zpp_{t+1,i}\right)\ln \frac{\zpp_{t+1,i}}{\zpp_{t,i}}\tag{\pref{eq: vomwu gradient}}\\
&\le\|\zp_{t+1}-\zp_{t}\|_1-\breg{\vent}(\zp_{t+1},\zp_{t})+\sum_{i\in\supp(\z^*)} z'_i\ln \frac{\zpp_{t+1,i}}{\zpp_{t,i}} \tag{$\supp(\z')\subseteq \supp(\z^*)$}\\
&\le\|\zp_{t+1}-\zp_{t}\|_1+\sum_{i\in\supp(\z^*)} \left|\ln \frac{\zpp_{t+1,i}}{\zpp_{t,i}}\right|\\
&\le\|\zp_{t+1}-\zp_{t}\|_1+\sum_{i\in\supp(\z^*)} \ln\left(1+\frac{\left| {\zpp_{t+1,i}}-{\zpp_{t,i}}\right|}{\min\{\zpp_{t+1,i}, \zpp_{t,i}\}}\right)   \\
&\le \|\zp_{t+1}-\zp_{t}\|_1+\sum_{i\in \supp(\z^*)}   \frac{\left| {\zpp_{t+1,i}}-{\zpp_{t,i}}\right|}{\min\{\zpp_{t+1,i}, \zpp_{t,i}\}} \tag{$\ln(1+a) \leq a$}\\
&\le \frac{2}{\epsv}\|\zp_{t+1}-\zp_{t}\|_1;\tag{\pref{lem: smallest index} and $\epsv\le 1$}
\end{align*}
on the other hand, the right hand side of \pref{eq: tmp eq} for \domwu is upper bounded by 
\begin{align*}
&\left(\nabla\dent(\zp_{t+1})-\nabla\dent(\zp_{t})\right)^{\top}\left(\z'-\zp_{t+1}\right)\\
&=\sum_i \alpha_i\left(z'_i-\zpp_{t+1,i}\right)\ln \frac{\qpp_{t+1,i}}{\qpp_{t,i}}\tag{\pref{eq: domwu gradient}}\\
&=\sum_i \alpha_iz'_i\ln \frac{\qpp_{t+1,i}}{\qpp_{t,i}}-\sum_i\alpha_i\zpp_{t+1,i}\ln \frac{\qpp_{t+1,i}}{\qpp_{t,i}}\\
&=\sum_{i\in\supp(\z^*)} \alpha_iz'_i\ln \frac{\qpp_{t+1,i}}{\qpp_{t,i}}-\sum_{h\in\csimp^\trplx}\alpha_h\zpp_{t+1,\sigma(h)}\sum_{j\in\seqactn_h}\qpp_{t+1,j}\ln \frac{\qpp_{t+1,j}}{\qpp_{t,j}}\\
&\le\|\balpha\|_\infty\sum_{i\in\supp(\z^*)} \left|\ln \frac{\qpp_{t+1,i}}{\qpp_{t,i}}\right|\tag{$\sum_{j\in\seqactn_h}\qpp_{t+1,j}\ln \frac{\qpp_{t+1,j}}{\qpp_{t,j}}\ge0$}\\
&\le\|\balpha\|_\infty\sum_{i\in\supp(\z^*)} \ln\left(1+\frac{\left| {\qpp_{t+1,i}}-{\qpp_{t,i}}\right|}{\min\{\qpp_{t+1,i}, \qpp_{t,i}\}}\right)   \\
&\le\|\balpha\|_\infty\sum_{i\in\supp(\z^*)}\frac{\left| {\qpp_{t+1,i}}-{\qpp_{t,i}}\right|}{\min\{\qpp_{t+1,i}, \qpp_{t,i}\}}  \tag{$\ln(1+a) \leq a$} \\
&\le \frac{\|\balpha\|_\infty}{\epsd}\sum_{i\in \supp(\z^*)}   {\left| {\qpp_{t+1,i}}-{\qpp_{t,i}}\right|}\tag{\pref{lem: smallest index}}.
\end{align*}
Since $\z'$ can be chosen as any point in $\calV^*(\calZ)$, we further lower bound the left-hand side of \pref{eq: tmp eq} using \pref{lem:them5prime} and get for \vomwu,
\begin{align}
    \eta C\|\z^*-\zp_{t+1}\|_1&\le \frac{2}{\epsv}\|\zp_{t+1}-\zp_{t}\|_1 + \frac{\|\z_t-\zp_{t+1}\|_1}{4}\nonumber \\
    &\le \frac{2}{\epsv}\left(\|\zp_{t+1}-\zp_{t}\|_1 +\|\z_t-\zp_{t+1}\|_1\right) \label{eq: 1 norm upper bound a},
\end{align}
and for \domwu,
\begin{align}
    \eta C\|\z^*-\zp_{t+1}\|_1& \le\frac{1}{4}\|\z_t-\zp_{t+1}\|_1+\frac{\|\balpha\|_\infty}{\epsd}\sum_{i\in\supp(\z^*)}{|\qpp_{t+1,i} - \qpp_{t,i}|}, 
     \tag{\pref{lem: smallest index}}\\
    &\leq \frac{1}{4}\|\z_t-\zp_{t+1}\|_1+\frac{\|\balpha\|_\infty}{\epsd} \sum_{i\in\supp(\z^*)} \frac{|\zpp_{t+1,i} - \zpp_{t,i}|}{\zpp_{t+1,p_i}}+\zpp_{t,i}\frac{|\zpp_{t+1,p_i} - \zpp_{t,p_i}|}{\zpp_{t+1,p_i}\zpp_{t,p_i}}\nonumber\\
    &\leq \frac{1}{4}\|\z_t-\zp_{t+1}\|_1+\frac{\|\balpha\|_\infty}{\epsd^2} \sum_{i\in\supp(\z^*)} {|\zpp_{t+1,i} - \zpp_{t,i}|}+{|\zpp_{t+1,p_i} - \zpp_{t,p_i}|} 
     \tag{\pref{lem: smallest index}}\\
     &\leq\frac{1}{4}\|\z_t-\zp_{t+1}\|_1+\frac{P\|\balpha\|_\infty}{\epsd^2}\|\zp_{t+1}-\zp_t\|_1\nonumber\\
    &\leq \left(\frac{1}{4}+\frac{P\|\balpha\|_\infty}{\epsd^2}\right)\left(\|\zp_{t+1}-\zp_t\|_1+ \|\z_t-\zp_{t+1}\|_1\right). \label{eq: 1 norm ub 1} 
\end{align}
Squaring both sides of \pref{eq: 1 norm upper bound a}, we get
\begin{align}
    \eta^2 C^2\|\z^*-\zp_{t+1}\|_1^2&\le\frac{4}{\epsv^2}\left(\|\zp_{t+1}-\zp_{t}\|_1 +\|\z_t-\zp_{t+1}\|_1\right)^2\nonumber\\
    &\le\frac{8}{\epsv^2}\left(\|\zp_{t+1}-\zp_{t}\|_1^2 +\|\z_t-\zp_{t+1}\|_1^2\right).\label{eq: 1 norm ub 2} 
\end{align}
Using the strong convexity of the regularizers, the left hand side of \pref{eq:eq6} can be bounded by
\begin{align*}
     &\KL(\zp_{t+1},\z_t)+\KL(\z_{t},\zp_t)\\
     &\geq \frac{1}{2}\|\zp_{t+1} - \z_t\|_1^2 + \frac{1}{2}\|\z_t-\zp_{t}\|_1^2\tag{$a^2+b^2\geq \frac{1}{2}(a+b)^2$}\\
     &\geq \frac{1}{8}\|\zp_{t+1} - \z_t\|_1^2 + \frac{1}{4}\left(\|\zp_{t+1} - \z_t\|_1^2+ \|\z_t-\zp_{t}\|_1^2 \right) \\
     &\geq \frac{1}{8}\left(\|\zp_{t+1}-\zp_t\|_1^2+\|\zp_{t+1}-\z_t\|_1^2\right)  \tag{$a^2+b^2\geq \frac{1}{2}(a+b)^2$ and triangle inequality}
\end{align*}
Combining this with \pref{eq: 1 norm ub 1} finishes the proof for \vomwu.
The similar argument works for \domwu by combining the inequality above with \pref{eq: 1 norm ub 2}.
\end{proof}


\subsection{Proofs of \pref{eq: vomwu pinsker} and \pref{eq: domwu pinsker}}\label{app:pinsker}
In this subsection, we prove \pref{eq: vomwu pinsker} and \pref{eq: domwu pinsker}.
We first show \pref{lem:reverse pinsker treeplex vomwu} and \pref{lem:reverse pinsker treeplex}.
Then we get \pref{eq: vomwu pinsker} and \pref{eq: domwu pinsker} by substituting $\z$ with $\zp_{t+1}$ in these lemmas.  
\begin{lemma}\label{lem:reverse pinsker treeplex vomwu}
    For any $\z\in\trplx$, we have  
    \begin{align*}\label{eq:reverse pinsker treeplex vomwu}
        \breg{\vent}(\z^*,\z)&\le \sum_{i\in\supp(\z^*)}\frac{(z^*_i-z_{i})^2}{\min_{i\in\supp(\z^*)}z_{i}}+\sum_{i\notin\supp(\z^*)}z_{i}\le\frac{3P\|\z^*-\z\|}{\min_{i\in\supp(\z^*)}z_{i}}.
    \end{align*}  
\end{lemma}
\begin{proof}
    By \pref{eq: 1d entropy}, we have
    \begin{align*}
        \breg{\vent}(\z^*,\z)&\le\sum_{i}\frac{(z^*_i-z_{i})^2}{z_{i}}\le\sum_{i\in\supp(\z^*)}\frac{(z^*_i-z_{i})^2}{\min_{i\in\supp(\z^*)}z_{i}}+\sum_{i\notin\supp(\z^*)}z_{i}\\
        &\le \frac{\|\z^*-\z\|^2}{\min_{i\in\supp(\z^*)}z_{i}}+\|\z^*-\z\|_1\le \frac{3P\|\z^*-\z\|}{\min_{i\in\supp(\z^*)}z_{i}},
    \end{align*}
    where the last inequality is because $\|\z^*-\z\|\le 2P$ and $\|\z^*-\z\|_1\le P\|\z^*-\z\|$.
\end{proof}
\begin{lemma}\label{lem:reverse pinsker treeplex}
    For any $\z\in\calZ$, we have
    \[
        \breg{\dent}(\z^*,\z)\le \|\balpha\|_\infty\left(\sum_{i\in\supp(\z^*)}\frac{4P}{z_{i}^*}\frac{\left({z^*_i-z_i}\right)^2}{q_i} +\sum_{i\notin\supp(\z^*)}z_{p_i }^*q_i\right)\le \frac{4P\|\balpha\|_\infty}{\min_{i\in \supp(\z^*)}z_i^*z_i}\|\z^*-\z\|_1.
    \]
\end{lemma}

\begin{proof}
    By direction calculation and \citep[Lemma 16]{wei2021linear}, we have
    \begin{align*}
        \breg{\dent}(\z^*,\z)&=\sum_{i}\alpha_{\infoi(i  )}z_{p_i}^*q^*_i\ln\left(\frac{q^*_i}{q_i}\right)\tag{\pref{eq:dent breg}}\\
        &\le \|\balpha\|_\infty\sum_{i}z_{p_i }^*\left(\one\{i\in\supp(\z^*)\}\frac{(q^*_i-q_i)^2}{q_i}+\one\{i\notin\supp(\z^*)\}q_i\right)\\
        &\le\|\balpha\|_\infty \sum_{i\in\supp(\z^*)}\frac{2z_{p_i }^*}{q_i}\left(\left(\frac{z^*_i-z_i}{z_{p_i}^*}\right)^2+\left(\frac{z_i}{z_{p_i}^*}-\frac{z_i}{z_{p_i}}\right)^2\right) +\|\balpha\|_\infty\sum_{i\notin\supp(\z^*)}z_{p_i }^*q_i\\
        &= \|\balpha\|_\infty\sum_{i\in\supp(\z^*)}\left(\frac{2}{q_iz_{p_i}^*}\left({z^*_i-z_i}\right)^2+\frac{2q_i}{z_{p_i }^*}\left({z_{p_i}^*}-{z_{p_i}}\right)^2\right) +\|\balpha\|_\infty\sum_{i\notin\supp(\z^*)}z_{p_i }^*q_i\\
        &\le \|\balpha\|_\infty\sum_{i\in\supp(\z^*)}\left(\frac{2}{z_{p_i}^*}\frac{\left({z^*_i-z_i}\right)^2}{q_i}+\frac{2P}{z_{i }^*}{\left({z^*_i-z_i}\right)^2}\right) +\|\balpha\|_\infty\sum_{i\notin\supp(\z^*)}z_{p_i }^*q_i\tag{$p_i\in\supp(\z^*)$ for all $i\in\supp(\z^*)$}\\
        &\le \|\balpha\|_\infty\sum_{i\in\supp(\z^*)}\left(\frac{4P}{z_{i}^*}\frac{\left({z^*_i-z_i}\right)^2}{q_i}\right) +\|\balpha\|_\infty\sum_{i\notin\supp(\z^*)}z_{p_i }^*q_i.
    \end{align*}  
    This proves the first inequality.
    The second equality in the lemma follows from
    \begin{align*}
        \breg{\dent}(\z^*,\z)&\le \|\balpha\|_\infty\sum_{i\in\supp(\z^*)}\left(\frac{4P}{z_{i}^*}\frac{\left({z^*_i-z_i}\right)^2}{q_i}\right) +\|\balpha\|_\infty\sum_{i\notin\supp(\z^*)}z_{p_i }^*q_i\\
        &\le \|\balpha\|_\infty\sum_{i\in\supp(\z^*)}\left(\frac{4P}{z_{i}^*z_i}\right) \left|{z^*_i-z_i}\right|+\|\balpha\|_\infty\sum_{i\notin\supp(\z^*)}\frac{z_{p_i }^*z_i}{z_{p_i}}\tag{$|z^*_i-z_i|\le 1$}\\
        &\le \|\balpha\|_\infty\sum_{i\in\supp(\z^*)}\left(\frac{4P}{z_{i}^*z_i}\right) \left|{z^*_i-z_i}\right|+\|\balpha\|_\infty\sum_{i\notin\supp(\z^*)}\frac{\one\{p_i\in\supp(\z^*)\}}{z_{p_i}}\cdot |z_i-0|\\
        &\le \frac{4P\|\balpha\|_\infty}{\min_{i\in \supp(\z^*)}z_i^*z_i}\|\z^*-\z\|_1.
    \end{align*}
\end{proof}
\subsection{Lower Bounds on the Probability Masses}\label{app:lower bounds}
In this subsection, we show for all $i\in\supp(\z^*)$ and $t$, $\zpp_{t,i}$ computed by \vomwu can be lower bounded by $\epsv$, while $\zpp_{t,i}$ and $z_{t,i}$ computed by \domwu can be lower bounded by $\epsd$, where $\epsv$ and $\epsd$ are defined in \pref{def: epsilon definition}.
We state the results in \pref{lem: smallest index vomwu} and \pref{lem: smallest index}, respectively.
We first state the stability of $\qp_t$ and $\q_t$, which directly follows from the stability of \omwu on simplex, for example, \citep[Lemma 17]{wei2021linear}.
\begin{lemma}\label{lem:stability}
For $\eta\le\frac{1}{8P}$, \ODMWU guarantees $\frac{3}{4}\qpp_{t,i}\le q_{t,i}\le \frac{4}{3}\qpp_{t,i}$ and $\frac{3}{4}\qpp_{t,i}\le \qpp_{t+1,i}\le \frac{4}{3}\qpp_{t,i}$. 
\end{lemma}

\begin{lemma}
    \label{lem: smallest index vomwu}
    For all $i\in\supp(\z^*)$ and $t$, \vomwu guarantees that $\zpp_{t,i}\geq \epsv$. 
\end{lemma}
\begin{proof}
    Using \pref{eq: simple recursion}, we have 
\begin{align}
    \KL(\z^*, \zp_t)\leq \Theta_t \leq \cdots \leq \Theta_1 = \tfrac{1}{16}\KL(\zp_1, \z_0)+\KL(\z^*, \zp_1)= \KL(\z^*, \zp_1), \label{eq:Theta_decreasing}
\end{align}
where the last equality is because $\zp_1=\z_0$. Thus, $\breg{\vent}(\z^*, \zp_t)\leq\breg{\vent}(\z^*, \zp_1)$.
Then, for any $i\in\supp(\z^*)$, we have 
\begin{align*}
z_i^*\ln \frac{1}{\zpp_{t,i}}
&\le \sum_{j}z_j^* \ln \frac{1}{\zpp_{t,j}} = \breg{\vent}(\z^*, \zp_t)  +\sum_{j}\left(z_j^*-\zpp_{t,j}-z_j^* \ln {z_j^*} \right)\\
&\leq \breg{\vent}(\z^*,\zp_1) -\sum_{j}z_j^* \ln {z_j^*}+\sum_j(z_j^*-\zpp_{t,j})\\
&\leq \sum_j z_j^* \ln \frac{1}{\zpp_{1,j}}+\sum_j(\zpp_{1,j}-\zpp_{t,j})\\
&\le  P+\sum_j z_j^* \ln \frac{1}{\zpp_{1,j}} =  P(1+\ln(1/\zmin)). 
\end{align*}
Therefore, we conclude for all $t$ and $i\in\supp(\z^*)$, $\zpp_{t,i}$ satisfies
\[
\zpp_{t,i}\ge \exp\left(-\frac{ P(1+\ln(1/\zmin))}{z_i^*}\right)\ge \min_{j\in\supp({\z^*})}  \exp\left(-\frac{ P(1+\ln(1/\zmin))}{z_j^*}\right)= \epsv. 
\]
\end{proof}

\begin{lemma}
     \label{lem: smallest index}
     For all $i\in\supp(\z^*)$ and $t$, \domwu guarantees that $\qpp_{t,i}\ge\zpp_{t,i}\geq \epsd$ and $q_{t,i}\ge z_{t,i}\geq \epsd$. 
\end{lemma}
\begin{proof}
Similar to \pref{lem: smallest index vomwu}, applying \pref{eq:Theta_decreasing} gives $\breg{\dent}(\z^*, \zp_t)\leq\breg{\dent}(\z^*, \zp_1)$.
Then, for any $i\in\supp(\z^*)$, we have 
\begin{align*}
z_i^*\ln \frac{1}{\qpp_{t,i}}
&\le \sum_{j}\alpha_jz_j^* \ln \frac{1}{\qpp_{t,j}} = \breg{\dent}(\z^*, \zp_t) - \sum_j \alpha_jz_j^* \ln q_j^*\leq \breg{\dent}(\z^*,\zp_1) - \sum_{j}\alpha_jz_j^* \ln {q_j^*} \\
&=  \sum_j\alpha_j z_j^* \ln \frac{1}{\qpp_{1,j}} \le  \|\balpha\|_\infty P\ln(1/\zmin). 
\end{align*}
Therefore, we conclude for all $t$ and $i\in\supp(\z^*)$, $\qpp_{t,i}$ satisfies
\[
\qpp_{t,i}\ge \exp\left(-\frac{\|\balpha\|_\infty P\ln(1/\zmin)}{z_i^*}\right)\ge \min_{j\in\supp({\z^*})}  \exp\left(-\frac{\|\balpha\|_\infty P\ln(1/\zmin)}{z_j^*}\right)
\]
and $\zpp_{t,i}$ satisfies
\[
    \zpp_{t,i}=\zpp_{t,p_i}\qpp_{t,i}=\zpp_{t,p_{p_i}}\qpp_{t,p_i}\qpp_{t,i}\ge\prod_{j\in\supp(\z^*)}\qpp_{t,j}\ge \min_{j\in\supp({\z^*})}  \exp\left(-\frac{\|\alpha\|_\infty P\ln(1/\zmin)}{z_j^*}\right)^P.  
\]
This finishes the first part of the proof.
Finally, using \pref{lem:stability}, we have for $i\in\supp(\z^*)$, $q_{t,i}\ge\frac{3}{4}\qpp_{t,i}\ge \epsd$ and
\begin{align*}
    z_{t,i}\ge\prod_{j\in\supp(\z^*)}q_{t,j}\ge \min_{j\in\supp({\z^*})}  \exp\left(-\frac{\|\alpha\|_\infty P^2\ln(1/\zmin)}{z_j^*}\right)\cdot\left(\frac{3}{4}\right)^P\ge\epsd.
\end{align*}
\end{proof}

\subsection{Proof of \pref{thm: vomwu rate}}\label{app:thm6}
Based on the results in the previous subsections and the discussion in \pref{sec:analysis of thm 6 7}, we can get $\Theta_t=\order(1/t)$ for both \vomwu and \domwu.
In this subsection, we formally state the results in \pref{thm: 1/t convergence for omwu} for both \vomwu and \domwu, and show the proof by combining all the components.
In particular, the result for \vomwu implies \pref{thm: vomwu rate}.

\begin{theorem}\label{thm: 1/t convergence for omwu}
    Under the uniqueness assumption, \vomwu and \domwu with step size $\eta\leq \frac{1}{8P}$ guarantee
    $
              \breg{\vent}(\z^*,\zp_t) \leq \frac{C_{13}}{t}
    $
    and
    $
              \breg{\dent}(\z^*,\zp_t) \leq \frac{C_{13}'}{t}
    $, respectively, 
        where $C_{13}',C_{13}>0$ are some constants depending on the game, $\zp_1$, and $\eta$. 

\end{theorem}

\begin{proof}
We start from \pref{lem:eq6}.
Using \pref{lem:reverse pinsker treeplex vomwu} and \pref{lem: smallest index vomwu}, the right hand side of \pref{eq:eq6} can be bounded by
\begin{align}\label{eq:app eq 7}
    \zeta_t &\ge C_{12}\|\z^*-\zp_{t+1}\|^2\ge \frac{\epsv^2C_{12}}{9P^2}\breg{\vent}(\z^*,\zp_{t+1})^2,
\end{align}
for \vomwu.
Similarly, using \pref{lem:reverse pinsker treeplex} and \pref{lem: smallest index}, we have for \domwu,
\begin{align*}
    \zeta_t &\ge C_{12}\|\z^*-\zp_{t+1}\|^2\ge \frac{\epsv^4C_{12}}{16P^2\|\balpha\|_\infty^2}\breg{\dent}(\z^*,\zp_{t+1})^2.
\end{align*}
On the other hand, applying \pref{eq: simple recursion} repeatedly, we get
\[
    \KL(\z^*,\zp_{1})=\Theta_{1}\ge \cdots\ge \Theta_{t+1}\ge\frac{1}{16}\KL(\zp_{t+1}, \z_{t}).
\]
Thus, $\zeta_t\ge\KL(\zp_{t+1}, \z_{t})\ge C_{10}\KL(\zp_{t+1}, \z_{t})^2$ for some $C_{10}>0$ depending on $\KL(\z^*,\zp_{1})$.
Combining this with \pref{eq:app eq 7} gives
\begin{align*}
\zeta_t&=\frac{1}{2}\left(\KL(\zp_{t+1}, \z_{t}) +\KL(\z_t, \zp_t)\right)+\frac{1}{2}\left(\KL(\zp_{t+1}, \z_{t}) +\KL(\z_t, \zp_t)\right)\\
&\ge \frac{1}{2}\cdot\frac{\epsv^2C_{12}}{ 9P^2}\breg{\vent}(\z^*,\zp_{t+1})^2+\frac{1}{2}\cdot C_{10}\breg{\vent}(\zp_{t+1}, \z_{t})^2\tag{\pref{eq:app eq 7}}\\
&\ge \frac{\epsv^2C_{12}}{ 36P^2}\cdot 2\breg{\vent}(\z^*,\zp_{t+1})^2+\frac{C_{10}}{4}\cdot 2\breg{\vent}(\zp_{t+1}, \z_{t})^2\\
&\ge\min\left\{\frac{\epsv^2C_{12}}{36P^2},\frac{C_{10}}{4}\right\}\Theta_{t+1}^2\ge C_{11}\Theta_{t+1}^2,
\end{align*}
for some $C_{11} >0$.
Similarly, we have $\zeta_t\ge C_{11}'\Theta_{t+1}^2$ for \domwu and constant $C_{11}'>0$.
Applying this to \pref{eq: simple recursion}, we obtain the recursion $\Theta_{t+1}\le \Theta_{t}-\frac{15}{16}C_{11}\Theta_{t+1}^2$.
This implies $\Theta_t\le\frac{C_{13}}{t}$ for some constant $C_{13}$ by \citep[Lemma 12]{wei2021linear}.
With the same argument, we can prove the case for \domwu.
\end{proof}
\subsection{Proof of \pref{thm: domwu rate}}\label{app:thm7}
\subsubsection{The Significant Difference Lemma}
In this subsection, we explain how to get the linear convergence result of \domwu.
As we discuss in the end of \pref{sec:analysis of thm 6 7}, this requires showing that $\sum_{i\notin\supp(\z^*)}z_{p_i }^*\qpp_{t+1,i}$ decreases significantly as $\zp_{t}$ gets close enough to $\z^*$.
The argument is shown in \pref{lem: KL sufficient decrease}.
Before that, we first show a lemma stating that for any information set $h\in\csimp^\trplx$, indices $i,j\in\seqactn_h$ such that $i\notin\supp(\z^*)$ and $j\in\supp(\z^*)$, $\widehat{L}_{t,i}$ is significantly larger than $\widehat{L}_{t,j}$ when $\zp_t$ is close to $\z^*$.
\begin{lemma}\label{lem: treeplex close to nash}
    Suppose $\|\z^*-\z\|\le \ca$. Then for $i\in\seqactn_h$, $h\in\csimp^\treeX$, we have
    \begin{align}\label{eq: treeplex close to nash}
        &\forall i \in \supp(\x^*), &&\lp_i\le V^*_{\infoi}+\frac{\eta\xi}{10};
        &&\forall i \notin \supp(\x^*), &&\lp_i\ge V^*_{\infoi}+\frac{9\eta\xi}{10},
    \end{align}
    where $\lp_i=(\G\y)_i+\sum_{\simp\in \csimp_i}-\frac{\alpha_\simp}{\eta}\ln\left(\sum_{j\in \seqactn_\simp}q_j\exp(-\eta \lp_j/\alpha_\simp)\right)$, $q_j=z_j/z_{p_j}$.
\end{lemma}

\begin{proof}
    We first consider terminal index $i$.
    By the assumption $\|\z^*-\z\|\le \ca$ and \pref{lem: gap lemma}, we have $\|\y-\y^*\|_1\le\frac{\eta\xi}{10P}$, and
    \begin{align}\label{eq: terminal nash supp}
        \lp_i = (\G\y)_i\le(\G\y^*)_i+\frac{\eta\xi}{10P}= V^*_{\infoi}+\frac{\eta\xi}{10P}
    \end{align}
    for $i\in\supp(\x^*)$ and 
    \begin{align*}
        \lp_i = (\G\y)_i\ge(\G\y^*)_i-\frac{\eta\xi}{10P}\ge V^*_{\infoi}+\xi-\frac{\eta\xi}{10P}\ge  V^*_{\infoi}+\frac{9\eta\xi}{10}
    \end{align*}
    for $i\notin\supp(\x^*)$ by the definition of $\xi$.
    Therefore, this shows \pref{eq: treeplex close to nash} for terminal indices. 
    In the following, we complete the proof by backward induction.
    Specifically, for nonterminal index $i\notin\supp(\x^*)$, we assume $\lp_j\ge  V^*_{\infoi(j)}+\frac{9\eta\xi}{10}$ for every \emph{descendant} $j$ (we say that index $j$ is a descendant of index $i$ if there exists a sequence of indexes $s_0,\dots,s_K$ for some $K>0$ such that $s_0=j$, $s_K=i$, and $p_{s_{k-1}}=s_{k}$ for every $k\in[K]$).
    Note that we always have $j\notin\supp(\x^*)$.
    We will prove $\lp_i$ satisfies \pref{eq: treeplex close to nash}, which completes the proof for $i\notin\supp(\x^*)$ by induction. 
    By assumption, we have
    \begin{align*}
        \lp_i&=(\G\y)_i+\sum_{\simp\in \csimp_i}-\frac{\alpha_\simp}{\eta}\ln\left(\sum_{j\in \seqactn_\simp}q_j\exp(-\eta \lp_j/\alpha_\simp)\right)\\
        &\ge (\G\y)_i+\sum_{\simp\in \csimp_i}-\frac{\alpha_\simp}{\eta}\ln\left(\max_{j\in\seqactn_\simp}\exp(-\eta \lp_j/\alpha_\simp)\right)\\
        &= (\G\y)_i+\sum_{\simp\in \csimp_i}\min_{j\in\seqactn_\simp}\lp_j\\
        &\ge(\G\y)_i+ \sum_{\simp\in \csimp_i}  V^*_{\simp}+\frac{9\eta\xi}{10}\tag{by the induction hypothesis}\\
        &\ge(\G\y^*)_i-\xi+ \sum_{\simp\in \csimp_i}  V^*_{\simp}+\frac{9\eta\xi}{10}\tag{$\|\y-\y^*\|\le\xi$}\\
        &\ge V^*_{\infoi(i)}+\frac{9\eta\xi}{10}.\tag{\pref{lem: gap lemma}}
    \end{align*}
    Similarly, for nonterminal index $i\in\supp(\x^*)$, we show for every descendant $j$,
    \begin{align}\label{eq: close to nash induction}
        \lp_j\le  V^*_{\infoi(j)}+\frac{\eta\xi}{10P}f(\infoi(j)),
    \end{align}
    where $f:\csimp^\treeX\to\R^+$ is defined recursively as follows. 
    For information set (simplex) $\simp$ such that $\seqactn_\simp$ contains terminal indices only, we let $f(\simp)=1$.
    Otherwise, we define 
    \begin{align}\label{eq:def of f}
      f(\simp)=\max_{k\in\seqactn_\simp}\sum_{s\in\csimp_k}\left(f(s)+\frac{1}{P}\right).
    \end{align} 
    This shows \pref{eq: treeplex close to nash} as \pref{lem: f upper bound} guarantees 
    \[
    f(h)\le (P-1)\cdot\left(1+\frac{1}{P}\right)< P,
    \]
    for every simplex $h$.
    It remains to prove \pref{eq:def of f} by induction.
    For the base case that $i$ is a terminal index, \pref{eq: close to nash induction} clearly holds by \pref{eq: terminal nash supp}. 
    For nonterminal index $i$, we have
    \begin{align}
        \lp_i&=(\G\y)_i+\sum_{\simp\in \csimp_i}-\frac{\alpha_\simp}{\eta}\ln\left(\sum_{j\in\seqactn_\simp}q_j\exp(-\eta \lp_j/\alpha_\simp)\right)\nonumber\\
        &\le (\G\y)_i+\sum_{\simp\in \csimp_i}-\frac{\alpha_\simp}{\eta}\ln\left(\sum_{j\in\seqactn_\simp\cap\supp(\x^*)}q_j\exp(-\eta \lp_j/\alpha_\simp)\right)\nonumber\\
        &\le (\G\y)_i+\sum_{\simp\in \csimp_i}-\frac{\alpha_\simp}{\eta}\ln\left[\exp\left(-\eta\left( V^*_{\simp}+\frac{\eta\xi}{10P}f(g)\right)/\alpha_\simp\right)\sum_{j\in\seqactn_\simp\cap\supp(\x^*)}q_j\right]\tag{by the assumption}\\
        &=(\G\y)_i+\sum_{\simp\in \csimp_i}\left[ V^*_{\simp}+\frac{\eta\xi}{20P}f(g)-\frac{\alpha_\simp}{\eta}\ln\left(\sum_{j\in\seqactn_\simp\cap\supp(\x^*)}q_j\right)\right]\label{eq: nonterminal nash}.
    \end{align}
    We continue to bound the last term. Let $c=\ca$. We have
    \begin{align*}
        -\frac{\alpha_\simp}{\eta}\ln\left(\sum_{j\in\seqactn_\simp\cap\supp(\x^*)}q_j\right)&=-\frac{\alpha_\simp}{\eta}\ln\left(\frac{\sum_{j\in\seqactn_\simp\cap\supp(\x^*)}x_j}{x_{p_j}}\right)\\
        &\le -\frac{\alpha_\simp}{\eta}\ln\left(\frac{\sum_{j\in\seqactn_\simp\cap\supp(\x^*)}(x_j^*-c)}{x_{p_j}^*+c}\right)\tag{$\|\z^*-\z\|\le \ca$}\\
        &\le -\frac{\alpha_\simp}{\eta}\ln\left(\frac{x_{p_j}^*-c|\seqactn_\simp|}{x_{p_j}^*+c}\right)\\
        &= -\frac{\alpha_\simp}{\eta}\ln\left(1-\frac{c(|\seqactn_\simp|+1)}{x_{p_j}^*+c}\right)\\
        &\le -\frac{\alpha_\simp}{\eta}\ln\left(1-\frac{cP}{\epsd}\right)\tag{$x_{p_j}^*+c\ge x_{p_j}^*\ge\epsd$ and $|\seqactn_\simp|+1\le P$}\\
        &\le -\frac{\alpha_\simp}{\eta}\ln\left(1-\frac{\eta^2\xi}{40\alpha_\simp P^2}\right)
        \tag{by definition of $c$}\\
        &\le\frac{\eta\xi}{20 P^2}.\tag{$-\ln(1-x)<2x$ for $0<x<0.5$}
    \end{align*}
    Plugging this back to the original inequalities, we get
    \begin{align*}
        \lp_i&\le(\G\y)_i+\sum_{\simp\in \csimp_i}\left( V^*_{\simp}+\frac{\eta\xi}{10P}f(g)+\frac{\eta\xi}{20 P^2}\right)\\
        &\le(\G\y^*)_i+\frac{\eta\xi}{20 P^2}+\sum_{\simp\in \csimp_i}\left(V^*_{\simp}+ \frac{\eta\xi}{10P}f(g)+\frac{\eta\xi}{20 P^2}\right)\tag{$\|\y-\y^*\|_1\le\frac{\eta\xi}{10P^2}$}\\
        &\le(\G\y^*)_i+\sum_{\simp\in \csimp_i}\left(V^*_{\simp}+ \frac{\eta\xi}{10P}f(g)+\frac{\eta\xi}{10 P^2}\right)\tag{$i$ is nonterminal}\\
        &=(\G\y^*)_i+\sum_{\simp\in \csimp_i}V^*_{\simp}+\frac{\eta\xi}{10P}\sum_{\simp\in \csimp_i}\left( f(g)+\frac{1}{P}\right)\\
        &\le V^*_{\infoi}+\frac{\eta\xi}{10P}\sum_{\simp\in \csimp_i}\left( f(g)+\frac{1}{P}\right)\tag{\pref{lem: gap lemma}}\\
        &\le  V^*_{\infoi}+\frac{\eta\xi}{10P}f(\infoi(i)),\tag{\pref{eq:def of f}}
    \end{align*}
    which shows \pref{eq: close to nash induction} by induction, and thus shows \pref{eq: treeplex close to nash}.
\end{proof}
\begin{lemma}\label{lem: f upper bound}
    Define $f:\csimp^\treeX\to\R^+$ as follows.
    \begin{align*}
    f(\simp)=\left\{\begin{aligned}
    &1, &&    \text{if $\seqactn_\simp$ contains terminal indices only;}\\
    &\max_{k\in\seqactn_\simp}\sum_{s\in\csimp_k}\left(f(s)+\frac{1}{P}\right), && \text{otherwise.}
    \end{aligned}\right.
    \end{align*}
    Then for every $\simp\in \csimp^\treeX$, we have
    \begin{align}
        f(\simp)\le I_\simp\left(1+\frac{1}{P}\right),\label{eq: f leq I g}
    \end{align}
    where $I_\simp$ is the number of indices that are the descendants of $\simp$ (we say that index $j$ is a descendant of simplex $\simp$ if $j\in\seqactn_\simp$ or index $j$ is a descendant of index $i$ for some $i\in\seqactn_\simp$).  
\end{lemma}
\begin{proof}
    If $\seqactn_\simp$ contains terminal indices only, since $I_\simp\ge 1$, \pref{eq: f leq I g} holds.
    Otherwise, suppose \pref{eq: f leq I g} holds for all simplexes that are descendants of $\simp$ (we say that simplex $h$ is a descendant of simplex $g$ if there exists a sequence of simplexes $s_0,\dots,s_K$ for some $K>0$ such that $s_0=h$, $s_K=g$, and $\seq({s_{k-1}})\in \seqactn_{s_{k}}$ for every $k\in[K]$). 
    We define
    \[
    k^*=\argmax_{k\in\seqactn_\simp}\sum_{s\in\csimp_k}\left(f(s)+\frac{1}{P}\right).
    \]
    Then we have
    \begin{align*}
        f(g)&= \sum_{s\in\csimp_{k^*}}\left(f(s)+\frac{1}{P}\right)\tag{by definition of $f$}\\
        &\le \sum_{s\in\csimp_{k^*}}\left[I_s\left(1+\frac{1}{P}\right)+\frac{1}{P}\right]\tag{by assumption}\\
        &\le 1+\sum_{s\in\csimp_{k^*}}\left[I_s\left(1+\frac{1}{P}\right)\right]\tag{$|\csimp_{k^*}|\le P$}\\
        &\le (I_g-1)\left(1+\frac{1}{P}\right)+{1}\\
        &\le I_\simp\left(1+\frac{1}{P}\right),
    \end{align*}
    where the third inequality is because $k^*$ is not a descendant of any $s\in \csimp_{k^*}$, and thus
    \[
    \sum_{s\in\csimp_{k^*}}I_s\le I_g-1.
    \]
    Therefore, we show \pref{eq: f leq I g} by induction.
\end{proof}

\subsubsection{The Counterpart of \pref{eq: bregman 2} for \domwu}
With \pref{lem: treeplex close to nash}, we can prove the following lemma, the counterpart of \pref{eq: bregman 2} for \domwu.
\begin{lemma}
    \label{lem: KL sufficient decrease}
    Under the uniqueness assumption, there exists a constant $C_{14}>0$  that depends on the game, $\eta$, and $\zp_1$ such that for any $t\geq 1$, \domwu with step size $\eta\leq \frac{1}{8P}$ guarantees
    \begin{align*}
        \breg{\dent}(\zp_{t+1}, \z_{t}) +\breg{\dent}(\z_t, \zp_t) \geq  C_{14}\breg{\dent}(\z^*,\zp_{t+1}) 
    \end{align*}
    as long as $\max\{\|\z^*-\zp_t\|_1, \|\z^*-\z_t\|_1\} \leq \ca$.
\end{lemma}
\begin{proof}
We define $\alphamin=\min_{h\in\csimp^\trplx}\alpha_h>0$.
Note that
    \begin{align}
    &\breg{\dent}(\zp_{t+1}, \z_t)+\breg{\dent}(\z_{t},\zp_t)\nonumber\\ 
    &=\sum_{\simp\in\csimp^\trplx}\alpha_g\cdot \zpp_{t+1,\seq(\simp)}\KLD(\qp_{t+1,\simp}, \q_{t,\simp})+\alpha_g\cdot z_{t,\seq(g)}\KLD(\q_{t,\simp},\qp_{t,\simp}) \tag{\pref{eq:dent breg}} \\
    &\ge\alphamin\sum_{\simp\in\csimp^\trplx}\zpp_{t+1,\seq(\simp)}\KLD(\qp_{t+1,\simp}, \q_{t,\simp})+z_{t,\seq(g)}\KLD(\q_{t,\simp},\qp_{t,\simp}) \nonumber \\
             &\ge\alphamin\epsd\sum_{\simp\in\csimp,~\seq(g)\in\supp(\z^*)}\KLD(\qp_{t+1,\simp}, \q_{t,\simp})+\KLD(\q_{t,\simp},\qp_{t,\simp}) \tag{\pref{lem: smallest index} }\\
             &\ge \frac{\alphamin\epsd}{3}\sum_{i\notin\supp(\z^*),p_i\in\supp(\z^*)} \left(\frac{(\qpp_{t+1,i}-q_{t,i})^2}{\qpp_{t+1,i}}+\frac{(q_{t,i}-\qpp_{t,i})^2}{q_{t,i}}\right)   \tag{\citep[Lemma 18]{wei2021linear}}\\
             &\ge \frac{\alphamin\epsd}{4}\sum_{i\notin\supp(\z^*),p_i\in\supp(\z^*)} \left(\frac{(\qpp_{t+1,i}-q_{t,i})^2}{\qpp_{t,i}}+\frac{(q_{t,i}-\qpp_{t,i})^2}{\qpp_{t,i}}\right)  \tag{\pref{lem:stability}} \\
             &\geq \frac{\alphamin\epsd}{8}\sum_{i\notin\supp(\z^*),p_i\in\supp(\z^*)} \frac{(\qpp_{t+1,i}-\qpp_{t,i})^2}{\qpp_{t,i}}.  \label{eq: second kind bound}
    \end{align}
    Below we continue to bound $\sum_{i\notin\supp(\z^*),p_i\in\supp(\z^*)} \frac{(\qpp_{t+1,i}-\qpp_{t,i})^2}{\qpp_{t,i}}$. 
    
    By the assumption, we have $\|\qp_t-\q^*\|_1\le P\|\zp_t-\z^*\|_1/\epsd^2 \leq \frac{\eta\xi}{10\alpha_i}$ and for index $i$ such that $i\notin\supp(\x^*)$ and $p_i\in\supp(\x^*)$, 
    \begin{align}
        \sum_{j\in\seqactn_{\infoi(i)},j\notin\supp(\x^*)}\qpp_{t,j} \leq \frac{\eta\xi}{10\alpha_i}.\label{eq:sum non-nash}
    \end{align}
        Moreover, by \pref{lem: treeplex close to nash}, we have (denote $\infoi=\infoi(i)$)
        \begin{align*}
             \qpp_{t+1,i} 
             &= \frac{\qpp_{t,i}\exp(-\eta \widehat{\lp}_i/\alpha_i)}{\sum_{j\in\seqactn_\infoi}\qpp_{t,j}\exp(-\eta \widehat{\lp}_j/\alpha_i)}\le \frac{\qpp_{t,i}\exp(-\eta \widehat{\lp}_i/\alpha_i)}{\sum_{j\in\seqactn_\infoi\cap\supp(\z^*)}\qpp_{t,j}\exp(-\eta \widehat{\lp}_j/\alpha_i)} \\
             &\le \frac{\qpp_{t,i}\exp(-\eta( V^*_{\infoi} + \frac{9\xi}{10})/\alpha_i)}{\sum_{j\in\seqactn_\infoi\cap\supp(\z^*)}\qpp_{t,j}\exp(-\eta ( V^*_{\infoi} + \frac{\xi}{10})/\alpha_i)}\tag{\pref{lem: treeplex close to nash}} \\
             &= \frac{\qpp_{t,i}\exp\left(-\frac{8}{10}\eta\xi/\alpha_i\right)}{\left(1-\sum_{j\in\seqactn_\infoi,j\notin \supp(\z^*)}\qpp_{t,j}\right)} \\
             &\leq \frac{\qpp_{t,i}\exp\left(-\frac{8}{10}\eta\xi/\alpha_i\right)}{\left(1-\frac{\eta\xi/\alpha_i}{10}\right)}\tag{\pref{eq:sum non-nash}}\\
             &\leq \qpp_{t,i}\left(1-\frac{1}{2}\frac{\eta \xi}{\alpha_i}\right).\tag{$\frac{\exp(-0.8u)}{1-0.1u}\leq 1-0.5u$ for $u\in[0,1]$}
        \end{align*}
        Rearranging gives
        \begin{align*}
             \frac{|\qpp_{t+1,i}-\qpp_{t,i}|^2}{\qpp_{t,i}}\geq \frac{\eta^2\xi^2}{4\|\balpha\|_\infty^2}\qpp_{t,i}\geq \frac{\eta^2\xi^2}{8\|\balpha\|_\infty^2}\qpp_{t+1,i},
        \end{align*}
        where the last step uses \pref{lem:stability}.
        The case for $\yp_t$ is similar.
        Combining this with \pref{eq: second kind bound}, we get 
        \begin{align}
        \breg{\dent}(\zp_{t+1},\z_t)+\breg{\dent}(\z_{t},\zp_t) 
        &\geq C_{12}'\sum_{i\notin\supp(\z^*),p_i\in\supp(\z^*)}\qpp_{t+1,i}\nonumber\\
        &\geq C_{12}'\sum_{i\notin\supp(\z^*)}z_{p_i}^*\qpp_{t+1,i} \label{eq:etaxix} 
        \end{align}
        for some $C_{12}'>0$.
        Now we combine two lower bounds of $\breg{\dent}(\zp_{t+1}, \z_t)+\breg{\dent}(\z_{t},\zp_t)$. Using \pref{lem:eq6} and \pref{eq:etaxix}, we get 
        \begin{align}
             &\breg{\dent}(\zp_{t+1}, \z_t)+\breg{\dent}(\z_{t},\zp_t)\nonumber\\
             &= \frac{1}{2}\left(\breg{\dent}(\zp_{t+1}, \z_t)+\breg{\dent}(\z_{t},\zp_t)\right) + \frac{1}{2}\left(\breg{\dent}(\zp_{t+1}, \z_t)+\breg{\dent}(\z_{t},\zp_t)\right)\nonumber \\
             &\geq 
            \frac{C_{12}}{2} \|\z^*-\zp_{t+1}\|_1^2 + \frac{C_{12}'}{2} \sum_{i\notin\supp(\z^*)}z_{p_i}^*\qpp_{t+1,i}.\label{eq: C12}
        \end{align}
    Also note that by \pref{lem:reverse pinsker treeplex}, we have
    \begin{align*}
        \breg{\dent}(\z^*,\zp_{t+1})&\le \|\balpha\|_\infty\left(\sum_{i\in\supp(\z^*)}\frac{4P}{z_{i}^*}\frac{\left({z^*_i-\zpp_{t+1,i}}\right)^2}{\qpp_{t+1,i}} +\sum_{i\notin\supp(\z^*)}z_{p_i }^*\qpp_{t+1,i}\right)\\
        &\le \frac{4\|\balpha\|_\infty P}{\epsd^2}\left(\sum_{i\in\supp(\z^*)}{\left({z^*_i-\zpp_{t+1,i}}\right)^2} +\sum_{i\notin\supp(\z^*)}z_{p_i }^*\qpp_{t+1,i}\right).\tag{\pref{lem: smallest index}}
    \end{align*}
    Combining this with \pref{eq: C12}, we conclude that 
    \begin{align*}
         \breg{\dent}(\zp_{t+1}, \z_t)+\breg{\dent}(\z_{t},\zp_t)&\ge\frac{\min\{C_{12},C_{12}'\}}{2}\left( \|\z^*-\zp_{t+1}\|_1^2 +  \sum_{i\notin\supp(\z^*)}z_{p_i}^*\qpp_{t+1,i}\right)\\
         &\ge \frac{\epsd^2}{8\|\balpha\|_\infty P}\min\left\{C_{12},C_{12}'\right\} \breg{\dent}(\z^{*},\zp_{t+1}),  
    \end{align*}
    which finishes the proof.
    
\end{proof}
\subsubsection{Proof of \pref{thm: domwu rate}}
With \pref{lem: KL sufficient decrease}, we are ready to prove \pref{thm: domwu rate}.
\begin{proof}[Proof of \pref{thm: domwu rate}]
Set $T_0=\frac{64C_{13}'}{c^2}$, where $c=\ca$. 
For $t\ge T_0$, we have using \pref{thm: 1/t convergence for omwu},
\begin{align*}
	\|\z^*-\zp_t\|_1^2&\le 2\breg{\dent}(\z^*, \zp_t)\le \frac{2C_{13}'}{T_0}\le c^2,\\
	\|\z^*-\z_t\|_1^2& \le 2\|\z^*-\zp_{t+1}\|_1^2+2\|\zp_{t+1}-\z_t\|_1^2\\ 
	&\le 4\breg{\dent}(\z^*,\zp_{t+1})+4\breg{\dent}(\zp_{t+1},\z_t) \\ 
	&\leq 64\Theta_{t+1}\le \frac{64C_{13}'}{T_0}\le c^2.
\end{align*}

Therefore, when $t\ge T_0$,  the condition of the second part of \pref{lem: KL sufficient decrease} is satisfied, and we have
\begin{align*}
	\zeta_t 
	&\geq \frac{1}{2}\breg{\dent}(\zp_{t+1}, \z_{t}) + \frac{1}{2}\zeta_t \\
	&\geq \frac{1}{2}\breg{\dent}(\zp_{t+1}, \z_{t}) + \frac{ C_{14}}{2}\breg{\dent}(\z^*, \zp_{t+1}) \tag{by \pref{lem: KL sufficient decrease}} \\
	&\geq C_{15}\Theta_{t+1}. 
\end{align*}
for some constant $C_{15}>0$.
Therefore, when $t\ge T_0$, $\Theta_{t+1}\le \Theta_t-\frac{15}{16}C_{15}\Theta_{t+1}$, which further leads to 
\[
\Theta_t \le \Theta_{T_0}\cdot \left(1+\frac{15}{16}C_{15}\right)^{T_0-t}\le \Theta_{1}\cdot \left(1+\frac{15}{16}C_{15}\right)^{T_0-t}=\breg{\dent}(\z^*, \zp_{1})\cdot \left(1+\frac{15}{16}C_{15}\right)^{T_0-t},
\]
where the second inequality uses \pref{eq:Theta_decreasing}.
The inequality trivially holds for $t< T_0$ as well, so it holds for all $t$. 
We finish the proof by relating $\breg{\dent}(\z^*,\z_t)$ and $\Theta_{t+1}$.
Note that by \pref{lem:reverse pinsker treeplex},
\begin{align*}
    \breg{\dent}(\z^*,\z_t)^2&\le \frac{16P\|\balpha\|_\infty^2}{\epsd^4}\|\z^*-\z_t\|_1^2\\
    &\le \frac{32P\|\balpha\|_\infty^2}{\epsd^4}\left(\|\z^*-\zp_{t+1}\|_1^2+\|\zp_{t+1}-\z_t\|_1^2\right)\\
    &\le \frac{1024P\|\balpha\|_\infty^2}{\epsd^4}\Theta_{t+1}.\\
\end{align*}
Therefore, we conclude
\begin{align*}
    \breg{\dent}(\z^*,\z_t)\le \sqrt{\frac{1024P\|\balpha\|_\infty^2}{\epsd^4}\Theta_{t+1}}\le \sqrt{\frac{1024P\|\balpha\|_\infty^2\breg{\dent}(\z^*, \zp_{1})}{\epsd^4} }\left(1+\frac{15}{16}C_{15}\right)^{\frac{T_0-t-1}{2}},
\end{align*}
which finishes the proof by setting
\[
C_3=\sqrt{\frac{1024P\|\balpha\|_\infty^2\breg{\dent}(\z^*, \zp_{1})}{\epsd^4} }\left(1+\frac{15}{16}C_{15}\right)^{\frac{T_0-1}{2}},~
C_4=\left(1+\frac{15}{16}C_{15}\right)^{\frac{1}{2}}-1.
\]
\end{proof}
\subsection{Remarks on \dogda}\label{app:dogda}
In this subsection, we discuss the technical difficulties to get a convergence rate for \dogda.
This is challenging even if we assume the uniqueness of the Nash equilibrium.
From the analysis of \vomwu and \domwu, we can see that \pref{lem: smallest index vomwu} and \pref{lem: smallest index} play an important role. 
The lemmas lower bound $\zpp_{t,i}$ with some game-dependent constants for $i$ in $\supp({\z^*})$, and the proofs are based on the observation that $\breg{\vent}(\z^*,\z)$ and $\breg{\dent}(\z^*,\z)$ approach infinity as $z_i$ approaches zero for $i\in\supp(\z^*)$. 
This property of the entropy regularizers, however, does not hold for the dilated Euclidean regularizer $\deuc$ in general.
Lower bounding $\zpp_{t,i}$ for \dogda could be possible when $\zp_{t}$ is sufficiently close to $\z^*$. 
For example, when
\[
\|\zp_{t}-\z^*\|\le \frac{1}{2}\min_{i\in\supp(\z^*)}z^*_i,
\]
we can lower bound $\zpp_{t,i}$ by $\frac{1}{2}\min_{i\in\supp(\z^*)}z^*_i$ for $i\in\supp(\z^*)$.
This must happen when $t$ is large by \pref{thm: asy convergence}, but the entire analysis will then depend on a potentially large ``asymptotic'' constant.
Therefore, even though we know that asymptotically, \dogda has linear convergence, getting a concrete rate as \vomwu and \domwu is still an open question. 

Another direction is to follow the analysis of \vogda, which gives a linear convergence rate by \pref{cor: vogda rate}.
However, in the analysis of \vogda, \citet{wei2021linear} implicitly use the fact that $\veuc$ is $\beta$-smooth, that is,
\[
\breg{\veuc}(\z,\z')\le\frac{\beta}{2}\|\z-\z'\|^2,
\]
for some $\beta>0$.
In fact, $\veuc$ is $1$-smooth.
This property, unfortunately, does not hold for $\deuc$.
We believe one can still show that $\deuc$ is $\beta$-smooth for some game-dependent $\beta$ once $\zp_{t}$ is sufficiently close to $\z^*$.
However, this again involves the asymptotic result in \pref{thm: asy convergence} and may prevent us from getting a concrete rate. 
In summary, using the existing techniques, we met some difficulties to obtain a concrete convergence rate for \dogda. 
However, \dogda performs well in the experiments.
Moreover, it reduces to \vogda in the normal-form games and achieves linear convergence in this case.
Therefore, we still believe that it is a promising direction to get a (linear) convergence rate for \dogda in theory.

\end{document}